\documentclass[letterpaper, 10 pt, journal, twoside]{IEEEtran}

\IEEEoverridecommandlockouts                              % This command is only needed if 
\usepackage{xcolor}
\usepackage{array}
\usepackage{mathtools}
\usepackage{hyperref}
\hypersetup{
    colorlinks=true,
    linkcolor=orange,
    filecolor=magenta,      
    urlcolor=orange,
    citecolor=orange,
}

\def\ours{Ease-In-Ease-Out fine-tuning}
\DeclareMathOperator{\supp}{supp}

\usepackage{times} % assumes new font selection scheme installed

\usepackage[utf8]{inputenc} % allow utf-8 input
\usepackage[T1]{fontenc}    % use 8-bit T1 fonts
\usepackage{hyperref}       % hyperlinks
\usepackage{url}            % simple URL typesetting
\usepackage{booktabs}       % professional-quality tables
\usepackage{amsfonts}       % blackboard math symbols
\usepackage{nicefrac}       % compact symbols for 1/2, etc.
\usepackage{microtype}      % microtypography

\usepackage{graphicx,psfrag,amsmath,amsfonts,verbatim}
\usepackage[small,bf]{caption}
\usepackage[utf8]{inputenc}

\usepackage{amsthm}
\usepackage{subfigure}
\usepackage{amsmath}
\usepackage{arydshln}

\usepackage[ruled]{algorithm2e}

\SetKw{Continue}{continue}
\SetKw{Break}{break}

\newtheorem{theorem}{Theorem}

\newtheorem{lemma}[theorem]{Lemma}
\newtheorem{proposition}[theorem]{Proposition}

\newtheorem{proposition1}[theorem]{Proposition}

\newtheorem{remark}[theorem]{Remark}
\theoremstyle{definition}
\newtheorem{definition}{Definition}[section]

\newcommand{\prg}[1]{\noindent\textbf{#1}.}

\title{Transfer Reinforcement Learning across Homotopy Classes}

\author{Zhangjie Cao$^{*1}$ and Minae Kwon$^{*2}$ and Dorsa Sadigh$^{3}$% <-this % stops a space
\thanks{Manuscript received October 15, 2020; revised January 6, 2021; accepted January 11, 2021.}% <-this % stops a space
\thanks{This paper was recommended for publication by Editor Dana Kulic upon evaluation of the Associate Editor and Reviewers' comments.}
\thanks{This work was supported by NSF Award Number 2006388 and the DARPA Hicon-Learn project}  %Use only for final RAL version
\thanks{$^{1}${\tt\small caozj@cs.stanford.edu}, $^{2}${\tt\small minae@cs.stanford.edu}, and $^{3}${\tt\small dorsa@stanford.edu}. The authors are with the Department of Computer Science, Stanford University, Stanford, CA 94305, USA}
\thanks{$*$ means equal contribution.}
\thanks{Digital Object Identifier (DOI): see top of this page.}
 }

\begin{document}

\maketitle

\begin{abstract}
The ability for robots to transfer their learned knowledge to new tasks---where data is scarce---is a fundamental challenge for successful robot learning.
While fine-tuning has been well-studied as a simple but effective transfer approach in the context of supervised learning, it is not as well-explored in the context of reinforcement learning.
In this work, we study the problem of fine-tuning in transfer reinforcement learning when tasks are parameterized by their reward functions, which are known beforehand. We conjecture that fine-tuning drastically underperforms when source and target trajectories are part of different \emph{homotopy classes}. We demonstrate that fine-tuning policy parameters across homotopy classes compared to fine-tuning within a homotopy class requires more interaction with the environment, and in certain cases is impossible. We propose a novel fine-tuning algorithm, \ours, that consists of a relaxing stage and a curriculum learning stage to enable transfer learning across homotopy classes. Finally, we evaluate our approach on several robotics-inspired simulated environments and empirically verify that the \ours\ method can successfully fine-tune in a sample-efficient way compared to existing baselines. 
\end{abstract}

\section{Introduction}
One of the goals of transfer learning is to efficiently learn policies in tasks where sample collection is cheap and then transfer the learned knowledge to tasks where sample collection is expensive.
Recent deep reinforcement learning (Deep RL) algorithms require an extensive amount of data, which can be difficult, dangerous, or even impossible to obtain~\cite{mnih2013playing,sallab2017deep,zhu2017target,murphy2014disaster}. Practical concerns regarding sample inefficiency make transfer learning a timely problem to solve, especially in the context of RL for robotics. Robots should be able to efficiently transfer knowledge from related tasks to new ones. For instance, consider an assistive robot that learns to feed a patient with a neck problem. The robot could not learn a sophisticated feeding policy when directly trained with a disabled patient in-the-loop, due to the limited number of interactions with the patient. Instead, the robot can learn how to feed abled-bodies, where it is easier to obtain data, and transfer the learned knowledge to the setting with the disabled patient using only a few samples.

We study transfer in the reinforcement learning setting where different tasks are parameterized by their reward function. While this problem and its similar variants have been studied using approaches like meta-RL~\cite{duan2016rl,wang2016learning, nagabandi2018learning,finn2017model,gupta2018meta}, multitask learning~\cite{ruder2017overview,teh2017distral}, and successor features~\cite{barreto2017successor}, fine-tuning as an approach for transfer learning in RL is still not well-explored. Fine-tuning is an important method to study for two reasons. First, it is a widely-used transfer learning approach that is very well-studied in supervised learning~\cite{mesnil2011unsupervised,hinton2006reducing,yosinski2014transferable}, but the limits of fine-tuning have been less studied in RL. Second, compared to peer approaches, fine-tuning does not require strong assumptions about the target domain, making it a general and easily applicable approach. \emph{Our goal is to broaden our understanding of transfer in RL by exploring when fine-tuning works, when it doesn't, and how we can overcome its challenges.} Concretely, we consider fine-tuning to be more efficient when it requires less interactive steps with the target environment. 

In this paper, we find that fine-tuning does not always work as expected when transferring between rewards whose corresponding trajectories belong to different homotopy classes. A homotopy class is traditionally defined as a class of trajectories that can be continuously deformed to one another without colliding with any barriers~\cite{bhattacharya2012topological}, see Fig.~\ref{fig:homotopy_figs}~(a). In this work, we generalize the notion of barriers to include any set of states that incur a large negative reward. These states lead to phase transitions (discontinuities) in the reward function. We assume that we know these barriers (and therefore homotopy classes) beforehand, which is equivalent to assuming knowledge of the reward functions. Knowing the reward function a-priori is a commonly made assumption in many robotics tasks, such as knowing goals~\cite{kober2013reinforcement,kalashnikov2018qt,plappert2018multi} or having domain knowledge of unsafe states beforehand~\cite{garcia2012safe,turchetta2016safe}. Also, reinforcement learning algorithms naturally assume that the reward function is available~\cite{sutton2018reinforcement}. Generalizing the notion of barriers allows us to go beyond robotics tasks classically associated with homotopy classes, e.g., navigation around barriers, to include tasks like assistive feeding. \emph{Our key insight is that fine-tuning continuously changes policy parameters and that leads to continuously deforming trajectories.} Hence, fine-tuning across homotopy classes will induce trajectories that intersect with barriers. This will introduce a high loss and gradients that point back to the source policy parameters. So it is difficult to fine-tune the policy parameters across homotopy classes. To address this challenge, we propose a novel \ours\ approach consisting of two stages: a Relaxing Stage and a Curriculum Learning Stage. In the Relaxing Stage, we relax the barrier constraint by removing it. Then, in the Curriculum Learning Stage, we develop a curriculum starting from the relaxed reward to the target reward that gradually adds the barrier constraint back.

The contributions of the paper are summarized as follows:
\begin{itemize}
    \item We introduce the idea of using homotopy classes as a way of characterizing the difficulty of fine-tuning in reinforcement learning. We extend the definition of homotopy classes to general cases and demonstrate that fine-tuning across homotopy classes requires more interaction steps with the environment than fine-tuning within the same homotopy class.
    \item We propose a novel \ours\ approach that fine-tunes across homotopy classes, and consists of a relaxing and a curriculum learning stage.
    \item We evaluate \ours~on a variety of robotics-inspired environments and show that our approach can learn successful target policies with less interaction steps than other fine-tuning approaches.
\end{itemize}

\section{Related Work}

\prg{Fine-tuning}
Fine-tuning is well-studied in the space of supervised learning~\cite{lange2012autonomous,finn2016deep,pinto2016supersizing,levine2016end,gupta2018robot,deng2009imagenet,sadeghi2018sim2real}. Approaches such as $L^2$-SP penalize the Euclidean distance of source and fine-tuned weights~\cite{xuhong2018explicit}. Batch Spectral Shrinkage penalizes small singular values of model features so that untransferable spectral components are repressed~\cite{chen2019catastrophic}. Progressive Neural Networks (PNN) transfer prior knowledge by merging the source feature into the target feature at the same layer~\cite{rusu2016progressive}. These works achieve state-of-the-art fine-tuning performance in supervised learning; however, directly applying fine-tuning methods to transfer RL does not necessarily lead to successful results as supervised learning and reinforcement learning differ in many factors such as access to labeled data or the loss function optimized by each paradigm~\cite{barto2004reinforcement}. We compare our approach with these fine-tuning methods for transfer RL.

In fine-tuning for robotics, a robot usually pre-trains its policy on a general source task, where there is more data available, and then fine-tunes to a specific target task. Recent work in vision-based manipulation shows that fine-tuning for off-policy RL algorithms can successfully adapt to variations in state and dynamics when starting from a general grasping policy~\cite{julian2020efficient}. As another example, RoboNet trains models on different robot platforms and fine-tunes them to unseen tasks and robots~\cite{dasari2019robonet}. A key difference is that our work proposes a systematic approach using homotopy classes for discovering when fine-tuning can succeed or fail. This is very relevant to existing literature in this domain, as our approach can explain \emph{why} a general policy, e.g., a general grasping policy, \emph{can} or \emph{cannot} easily be fine-tuned to more specific settings. 

\prg{Transfer Reinforcement Learning} There are several lines of work for transfer RL including successor features, meta-RL and multitask learning. We refer the readers to~\cite{taylor2009transfer} for a comprehensive survey. We compare these works to our approach below.

\textit{Successor Features}. Barreto et al., address the same reward transfer problem as ours by learning a universal policy across tasks based on successor features~\cite{barreto2017successor}. However, this work makes a number of assumptions about the structure of the reward function and requires that the rewards between source and target tasks be close to each other, while our work has no such constraints.

\textit{Meta-RL}.
Meta learning provides a generalizable model from multiple (meta-training) tasks to quickly adapt to new (meta-test) tasks. There are various Meta RL methods including RNN-based~\cite{duan2016rl,wang2016learning,nagabandi2018learning}, gradient-based~\cite{finn2017model,gupta2018meta, nagabandi2018learning}, or meta-critic approaches~\cite{sung2017learning}. The gradient-based approach is the most related to our work, which finds policy parameters (roughly akin to finding a source task) that enable fast adaptation via fine-tuning. Note that all meta-RL approaches assume that agents have access to environments or data of meta-training tasks, which is not guaranteed in our setting. Here our focus is to discover when fine-tuning is generally challenging based on homotopy classes. In our experiments we compare our algorithm to core fine-tuning approaches rather than techniques that leverage ideas from fine-tuning or build upon them.

\textit{Multitask learning}. Other works transfer knowledge by simultaneously learning multiple tasks or goals~\cite{ruder2017overview}. In these works, transfer is enabled by learning shared representations of tasks and goals~\cite{deisenroth2014multi,teh2017distral,ruder2017overview, nair2018visual}. In our work, we consider the setting where tasks are learned sequentially.

\textit{Regularization}. Cobbe et al's work~\cite{cobbe2019quantifying} proposes a metric to quantify the generalization ability of RL algorithms and compare the effects of different regularization techniques on generalization. The paper compares the effects of deeper networks, batch normalization, dropout, L2 Regularization, data augmentation and stochasticity ($\epsilon$-greedy action selection and entropy bonus). The proposed techniques are designed for general purpose transfer reinforcement learning but are not specially designed for transfer reinforcement learning across homotopy classes. We compare our approach against using deeper networks, dropout, and entropy bonuses in our Navigation and Lunar Lander experiments and show that these techniques alone are not sufficient to transfer across homotopy classes (see supplementary materials).

\section{Fine-tuning across Homotopy Classes}\label{sec:difficulty}
In transfer reinforcement learning, our goal is to fine-tune from a source task to a target task. We formalize a task using a Markov Decision Process $\mathcal{M}= \langle \mathcal{S}, \mathcal{A}, p, \mathcal{R}, \rho_0, \gamma \rangle$, where $\mathcal{S}$ is the state space, $\mathcal{A}$ is the action space, $p: \mathcal{S} \times \mathcal{A} \times \mathcal{S}\rightarrow [0,1] $ is the transition probability, $\rho_0$ is the initial state distribution, $\mathcal{R}: \mathcal{S} \times \mathcal{A} \rightarrow \mathbb{R}$ is the reward function, and $\gamma$ is the discount factor. We denote $\mathcal{M}_s$ as the source task and $\mathcal{M}_t$ as the target task. We assume that $\mathcal{M}_s$ and $\mathcal{M}_t$ only differ on reward function, i.e., $\mathcal{R}_s \neq \mathcal{R}_t$. These different reward functions across the source and target task can for instance capture different preferences or constraints of the agent. A stochastic policy $\pi: \mathcal{S} \times \mathcal{A} \rightarrow [0,1]$ defines a probability distribution over the action in a given state. The goal of RL is to learn an optimal policy $\pi^*$, which maximizes the expected discounted return $\eta_{\pi}=\mathbb{E}_{\xi \sim \pi}\large[G_\tau(\xi)\large] = \mathbb{E}_{s_0 \sim \rho_0,\pi}\left[\sum_{\tau=0}^{\infty}\gamma^{\tau} \mathcal{R}(s_{\tau}, a_{\tau})\right]$. We define a trajectory to be the sequence of states the agent has visited over time $\xi=\{s_0, s_1, \dots \}$, and denote $\xi^*$ to be a trajectory produced by the optimal policy, $\pi^*$. We assume that the optimal policy for the source environment $\pi^*_s$ is available or can be easily learned. Our goal is then to leverage knowledge from $\pi^*_s$ to learn the optimal policy $\pi^*_t$ for task $\mathcal{M}_t$. We aim to learn $\pi^*_t$ with substantially fewer training samples than other comparable fine-tuning approaches. 

\begin{figure}
    \centering
    \includegraphics[width=\columnwidth, height=32mm]{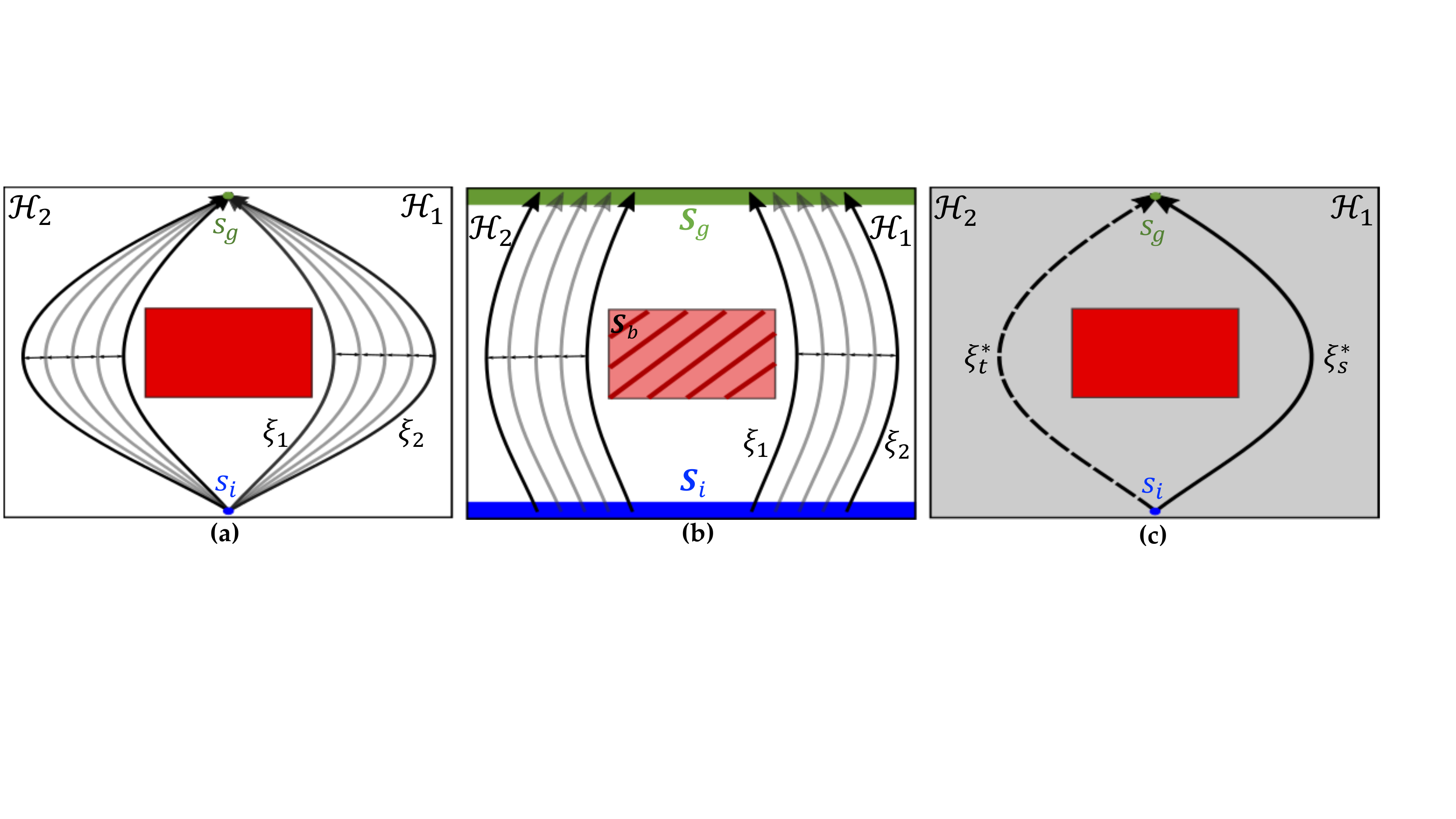}
    \caption{\textbf{(a)} Homotopy classes. $\xi_1$ and $\xi_2$ are part of the same homotopy class because they can be continuously deformed into each other. \textbf{(b)} Generalized homotopy classes with expanded definitions of start, end, and barrier states. \textbf{(c)} Fine-tuning problem from \emph{left} side to the \emph{right} side. The goal is to find $\pi_t^*$ that produces $\xi_t^*$.}
    \label{fig:homotopy_figs}
    \vspace{-15pt}
\end{figure}

\subsection{Homotopy Classes}
Homotopy classes are formally defined by homotopic trajectories in navigation scenarios in \cite{bhattacharya2012topological}:
\begin{definition}
\textbf{Homotopic Trajectories and Homotopy Class.} Two trajectories $\xi_1, \xi_2$ connecting the same initial and end points $s_i, s_g$ are homotopic if and only if one can be continuously deformed without intersecting with any barriers. Homotopic trajectories are clustered into a homotopy class.~\footnote{Even though the presence of a single obstacle introduces infinitely many homotopy classes, in most applications we can work with a finite number of them, which can be formalized by the concept of $Z_2$-homology~\cite{chen2010measuring}. For algorithms that compute these homology classes see~\cite{chen2010measuring}.}
\end{definition}
Fig.~\ref{fig:homotopy_figs}~(a) illustrates a navigation scenario with two homotopy classes $\mathcal{H}_1$ and $\mathcal{H}_2$ separated by a red barrier. $\xi_1$ and $\xi_2$ can be continuously deformed into each other without intersecting the barrier, and hence are in the same homotopy class. 

\prg{Generalization} The original definition of homotopy classes is limited to navigation scenarios with deterministic trajectories and the same start and end states. We generalize this definition to encompass a wider range of tasks in three ways.

Firstly, we account for tasks where there could be more than one feasible initial and end state. We generalize the initial and end points $s_i, s_g$ to a set of states $\mathbf{S}_i$ and $\mathbf{S}_g$, where $\mathbf{S}_i$ contains all the possible starting states and $\mathbf{S}_g$ contains all possible ending states as shown in Fig.~\ref{fig:homotopy_figs}~(b). 

Secondly, we generalize the notion of a barrier to be a set of states that are penalized with large negative rewards $\mathbf{S}_b = \{s | \mathcal{R}(s,a) \leftarrow \mathcal{R}'(s,a) - M\}$, where $M$ is a large positive number and $\mathcal{R}'(s,a)$ is the reward without any barriers. Large negative rewards correspond to any negative phase transitions or discrete jumps in the reward function. Importantly, the generalized `barrier' allows us to define homotopy classes in tasks without physical barriers that penalize states with large negative rewards (see our Assistive Feeding experiment). Although source and target tasks differ in reward functions, they share the same barrier states.%, i.e., for all $s\in \mathbf{S}_b$, $\mathcal{R}_s(s,\cdot)-\mathcal{R}_t(s,\cdot) = 0$.

Thirdly, we need to generalize the notion of continuously deforming \emph{trajectories} to \emph{trajectory distributions} when considering stochastic policies. We appeal a distribution distance metric, Wasserstein-$\infty$ ($W_\infty$) metric, that penalizes \emph{jumps} (discontinuities) between trajectory distributions induced by stochastic policies. We can now define our generalized notion of homotopic trajectories.

\begin{definition}
\textbf{General Homotopic Trajectories.} Two trajectories $\xi_1, \xi_2$ with distributions $\mu_1$ and $\mu_2$ and with the initial states $s_i \in \mathbf{S}_i$ and the final states $s_g \in \mathbf{S}_g$ are homotopic if and only if one can be continuously deformed into the other in the $W_\infty$ metric without receiving large negative rewards. Definitions for the $W_\infty$ metric and $W_\infty$-continuity are in Section I of the supplementary materials.
\end{definition}

General homotopic trajectories are depicted in Fig.~\ref{fig:homotopy_figs}~(b). The generalized definition of a homotopy class is the set of general homotopic trajectories. Note that using the $W_\infty$ metric is crucial here. Homotopic equivalence of stochastic policies according to other distances like total variation, KL-divergence, or even $W_1$ is usually trivial because distributions that even have a tiny mass on all deterministic homotopy classes become homotopically equivalent. On the other hand, in the $W_\infty$ metric, the distance between distributions that tweak the percentages, even by a small amount, would be at least the minimum distance between trajectories in different deterministic homotopy classes, which is a constant. So to go from one distribution over trajectories to another one with different percentages, one has to make a \emph{jump} according to the $W_\infty$ metric.

\subsection{Challenges of Fine-tuning across Homotopy Classes}
\prg{Running Example} We explain a key optimization issue caused by barriers when fine-tuning across homotopy classes. We illustrate this problem in Fig.~\ref{fig:homotopy_figs}~(b). An agent must learn to navigate to its goal $s_g \in \mathbf{S}_g$ without colliding with the barrier. Assuming that the agent only knows how to reach $s_g$ by swerving right, denoted by $\xi_s^*$, we want to learn how to reach $s_g$ by swerving left (i.e., find $\pi_t^*$). 

We show how the barrier prevents fine-tuning from source to target in Fig.~\ref{fig:surface}. This figure depicts the loss landscape for the target task with and without barriers. All policies are parameterized by a single parameter $\theta \in \mathbb{R}^2$ and optimized with the vanilla policy gradient algorithm~\cite{baxter2001infinite}. Warmer regions indicate higher losses in the target task whereas cooler regions indicate lower losses. 

Policies that collide with barriers cause large losses shown by the hump in Fig.~\ref{fig:surface}~(b). Gradients point away from this large loss region, so it is difficult to cross the hump without a sufficiently large step size. In contrast, in Fig.~\ref{fig:surface}~(a), the loss landscape without the barrier is smooth, so fine-tuning is easy to converge. Details on the landscape plots are in Section IV of the supplementary materials. 

\begin{figure*}
    \centering
    \includegraphics[ height=35mm]{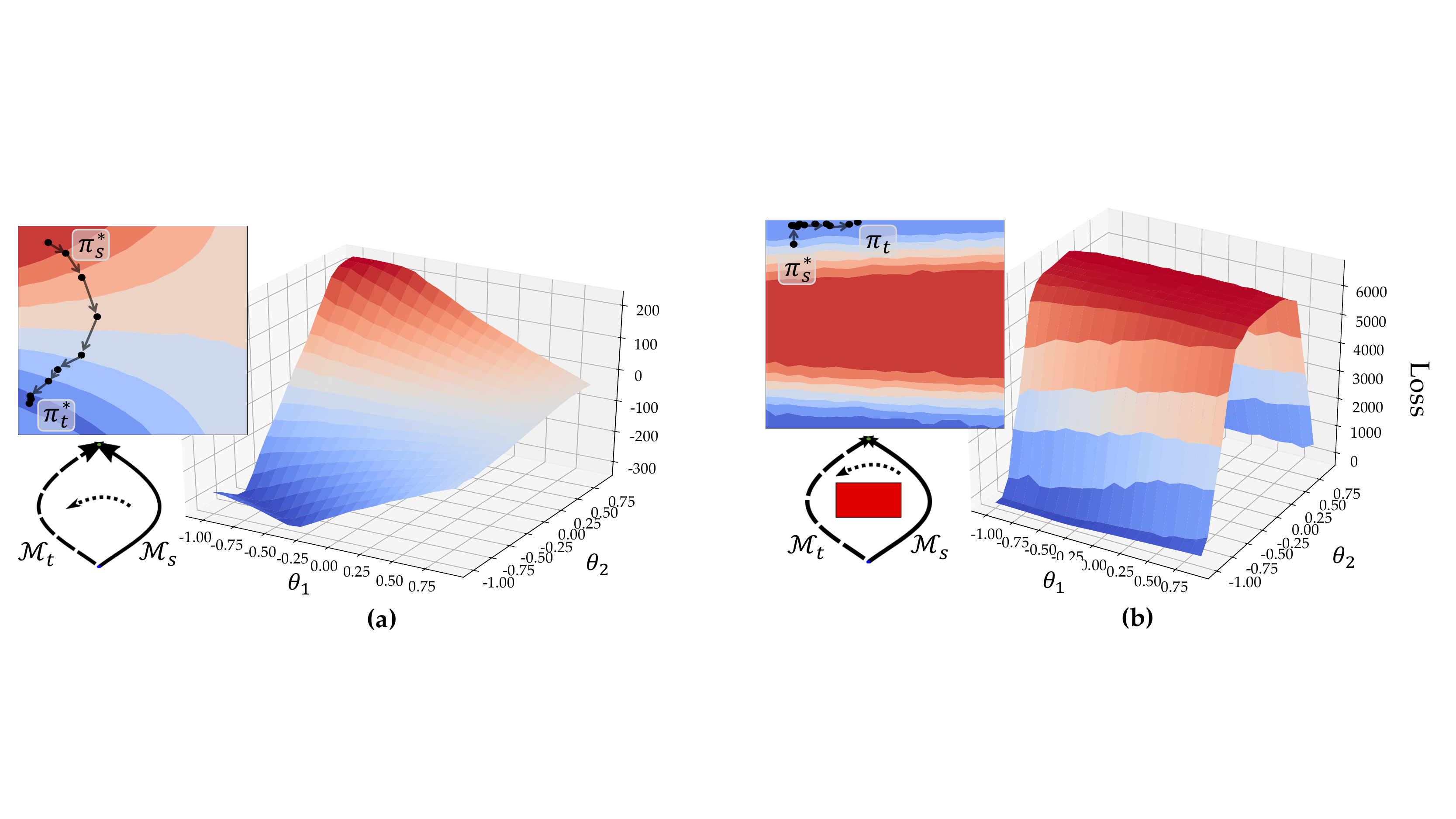}
    \caption{\textbf{(a)} Loss landscape of our running example without a barrier. The top-down pictures illustrate the gradient steps taken when fine-tuning from source to target tasks. \textbf{(b)} Loss landscape with a barrier. Barriers create gradients away from it that make it difficult to fine-tune from source to target tasks. }
    \label{fig:surface}
    \vspace{-15pt}
\end{figure*}

We now formally investigate how discontinuities in trajectory space caused by barriers affect fine-tuning of model-free RL algorithms. We let the model parameterized by $\theta$ to induce a policy $\pi_{\theta}$, and define the loss for the model to be $\ell(\theta)$. We assume that $\ell(\theta)$ is high when the expected return $\eta_{\pi_{\theta}} = \mathbb{E}_{\xi \sim \pi_{\theta}}[G(\xi)]$ is low. This assumption is satisfied in common model-free RL algorithms such as vanilla policy gradient. We optimize our policy using gradient descent with step size $\alpha$: $\theta_{k+1} = \theta_k - \alpha \nabla_{\theta}\ell(\theta)|_{\theta_k}$. We can now define what it means to fine-tune from one task to another. 

Let $\theta^*_{s}$ be the optimal set of parameters that minimizes the cost function on the source task. Using $\ell^t(\theta)$ as the loss for target reward, fine-tuning from $\mathcal{M}_s$ to $\mathcal{M}_t$ for $n$ gradient steps is defined as: $\theta_1 \leftarrow \theta_s^*$ and $\theta_{k+1} = \theta_k - \alpha \nabla_{\theta}\ell^t(\theta)|_{\theta_k}$ for $k = 1,\dots,n.$

We consider a policy to have successfully fine-tuned to $\mathcal{M}_t$ if the received expected return is less than $\epsilon$ away from the expected reward of the optimal target policy $\pi_t^*$ for some small $\epsilon$, i.e.,
$|\eta^t_{\pi_{\theta}} - \eta^{t}_{\pi_t^*}| < \epsilon$.

We now theoretically analyze why it is difficult to fine-tune across homotopy classes. Due to the space limit, we only include our main theorem and remark in the paper. We refer readers to the supplementary materials for the proofs.

\begin{definition}\textbf{$W_\infty$-continuity of policy.} 
    A policy $\pi_\theta$ parameterized by $\theta$ is $W_\infty$-continuous if the mapping $(\theta)\mapsto \pi_\theta(s, a)$, which maps a vector of parameters in a metric space to a distribution over state-actions is continuous in $W_\infty$ metric.
\end{definition}

\begin{definition}\textbf{$W_\infty$-continuity of transition probability function.} 
An MDP $\mathcal{M}$ with transition probability function $p$ is called $W_\infty$-continuous if the mapping $(s, a)\mapsto p(s, a, \cdot)$ which maps a state-action pair in a metric space to a distribution over states is continuous in $W_\infty$ metric.
\end{definition}

\begin{theorem}\label{theorem:1}
Assume that $\pi_\theta$ is a parametrized policy for an MDP $\mathcal{M}$. If both $\pi_\theta$ and $\mathcal{M}$ are $W_\infty$-continuous, then a continuous change of policy parameters $\theta$ results in a continuous deformation of the induced random trajectory in the $W_\infty$ metric.
 However, continuous deformations of the trajectories do not ensure continuous changes of their corresponding policy parameters.
\end{theorem}

Note that the theorem also applies to deterministic policies. For deterministic policies $W_\infty$-continuity is the same as the classical notion of continuity. Theorem~\ref{theorem:1} bridges the idea of changes in policy parameters with trajectory deformation. To use this theorem, we need assumptions on the learning rate $\alpha$ and bounds on the gradients.
Specifically for any $\theta_1$ and $\theta_2$ induced by policies in two different homotopy classes, we need to assume: $\lVert \theta_1-\theta_2 \rVert>\alpha \max(\nabla_{\theta}\ell^t(\theta)|_{\theta_1}, \nabla_{\theta}\ell^t(\theta)|_{\theta_2})$. With such small enough learning rate $\alpha$, fine-tuning will always induce trajectories that visit barrier states, $ \mathbf{S}_b$. 

\begin{remark}\label{remark:2}
Intuitively, the conclusion we should reach from Theorem~\ref{theorem:1} is that fine-tuning model parameters across homotopy classes is more difficult or even infeasible in terms of number of interaction steps in the environment compared to fine-tuning within the same homotopy class; this is under the assumptions that the transition probability function and policy of $\mathcal{M}$ are $W_\infty$-continuous, learning rate is sufficiently small, and gradients are bounded \footnote{Modern optimizers and large step sizes can help evade local minima but risk making training unstable when step sizes are too large.}.
\end{remark}
 
\section{\ours\ Approach}
\label{sec:approach}
Our insight is that even though there are states with large negative rewards that make fine-tuning difficult across homotopy classes, there is still useful information that can be transferred across homotopy classes. Specifically, we first \emph{ease in} or \emph{relax} the problem by removing the negative reward associated with barriers, which enables the agent to focus on fine-tuning towards target reward without worrying about large negative rewards. We then \emph{ease out} by gradually reintroducing the negative reward via a \emph{curriculum}. We assume the environment is alterable in order to remove and re-introduce barrier states. In most cases, this requires access to a simulator, which is a common assumption in many robotics applications~\cite{dosovitskiy2017carla,erickson2020assistivegym,brockman2016openai,shen2020igibson}. We assume that during the relaxing stage as well as each subsequent curriculum stage, we are able to converge to an approximately optimal policy for that stage using reinforcement learning.

\begin{figure*}
    \centering
    \includegraphics[height=25mm]{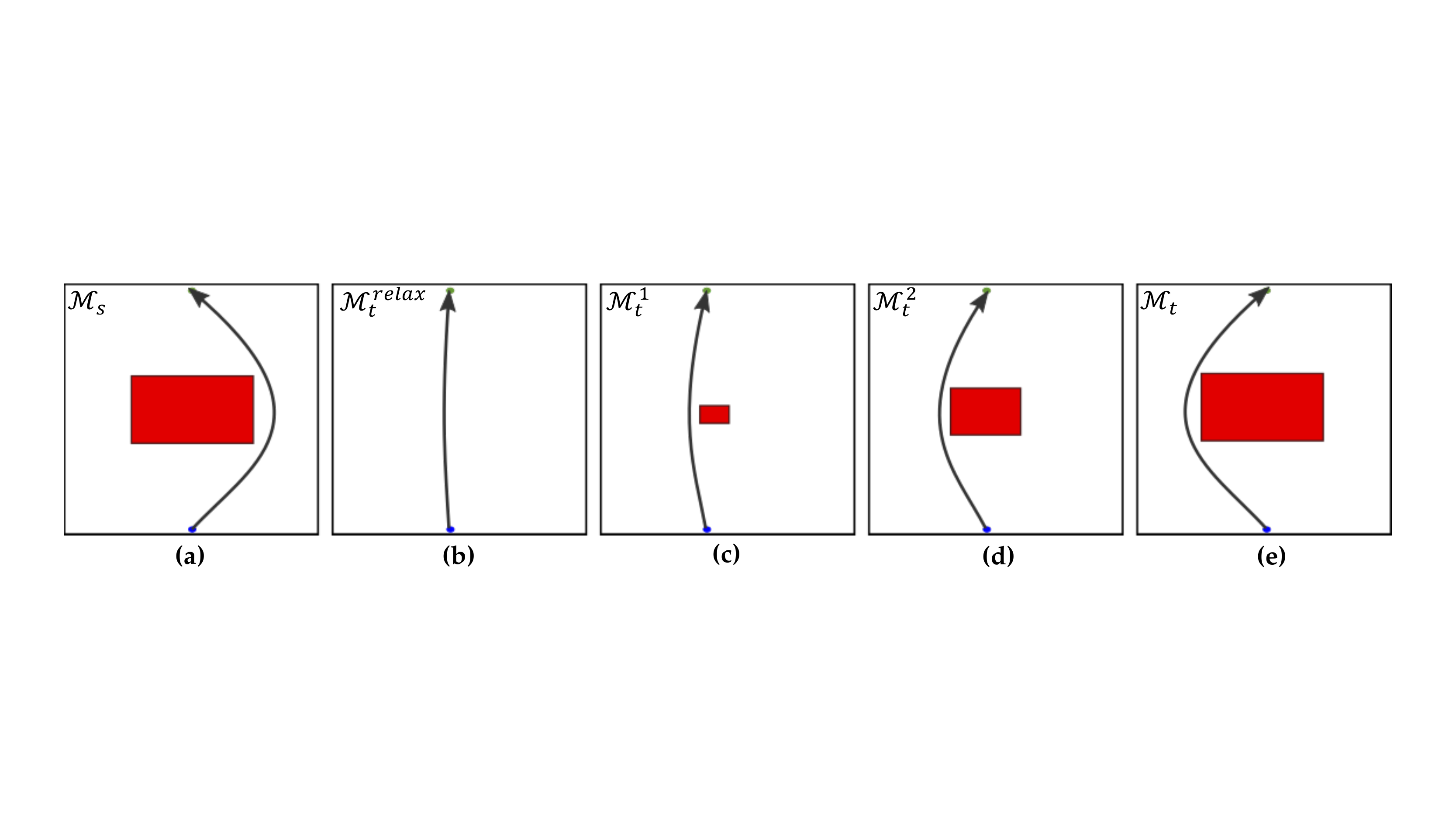}
    \caption{\ours \ approach for the single barrier set case. The red represents the negative reward associated with the barrier. \textbf{(a)} Source task. \textbf{(b)} Relaxing stage. The resulting policy produces trajectories that lean toward the left but remain fairly centered. 
    \textbf{(c)-(e)} Curriculum learning stage with $K=3$. We introduce larger subsets of the barrier states set $\mathbf{S}_b$ and fine-tune. This results in trajectories that are slowly pushed towards the left.}
    \label{fig:alg}
    \vspace{-10pt}
\end{figure*}

\prg{Ease In: Relaxing Stage}
In the relaxing stage, we remove the barrier penalty in the reward function, i.e., $\forall s \in \mathbf{S}_b, \mathcal{R}^{\text{relax}}_t(s,a) = \mathcal{R}'(s,a)$. We denote the target MDP with relaxed reward function as $\mathcal{M}^{\text{relax}}_t$. Note that we do not physically remove the barriers, so the transition function does not change. We start from $\pi_s^*$ and train the policy in $\mathcal{M}_t^{\text{relax}}$ to obtain $\pi^*_{\text{relax}}$. The relaxation removes large losses incurred by the barriers, making fine-tuning much easier than na\"ive fine-tuning.

\prg{Ease Out: Curriculum Learning Stage}
The relaxing stage finds an optimal policy $\pi_{\text{relax}}^*$ for $\mathcal{M}^{\text{relax}}_{t}$. We now need to learn the optimal policy for original target MDP $\mathcal{M}_t$ that actually penalizes barrier states with a large penalty $-M$. We develop two curricula to gradually introduce this penalty. 

\noindent\textit{(1) Reward Weight (general case)}. We design a general curriculum that can be used for any environments by gradually increasing the penalty from 0 to $M$ using a series of values $\alpha_1,\dots,\alpha_K$ satisfying $0<\alpha_1<\alpha_2<\cdots<\alpha_K=1$. We redefine our reward function to include intermediary values:
\begin{equation}
\begin{small}
    \mathcal{R}^{\text{cur}}(s,a;\alpha_k) = 
    \begin{cases} 
       \mathcal{R}'(s,a) - \alpha_k M & s \in \mathbf{S}_b \\
       \mathcal{R}'(s,a) & s \not\in  \mathbf{S}_b
   \end{cases}
  \end{small}
\end{equation}
This allows us to define a sequence of corresponding tasks $\mathcal{M}_t^0, \dots, \mathcal{M}_t^K$ where $\mathcal{M}_t^0 \equiv \mathcal{M}^{\text{relax}}_t$ and $\mathcal{M}_t^K \equiv \mathcal{M}_t$. For each new task $\mathcal{M}_t^k$, we initialize the policy with the previous task's optimal policy $\pi_{\theta_{k-1}}^*$ and train it using the reward $\mathcal{R}^{\text{cur}}(s,a; \alpha_k)$. The detailed algorithm is shown in Algorithm 1 in Section III of the supplementary materials.

\noindent\textit{(2) Barrier Set Size.}
When there is only a single barrier set $\mathbf{S}_b$ (i.e., $\mathbf{S}_b$ is connected), we can also build a curriculum around the set itself. Here, we keep the $-M$ penalty but gradually increase the set of states that incur this penalty. We can guarantee that our algorithm always converges as we discuss in our analysis section below. 

To build a curriculum, we can choose any state $s \in \mathbf{S}_b$ as our initial set and gradually inflate this set to $\mathbf{S}_b$ by connecting more and more states together \footnote{A connected path is defined differently for continuous and discrete state spaces. For example, in continuous state spaces, a connected path means a continuous path.}. For example, we can connect new states that are within some radius of the current set. This allows us to define a series of connected sets $\mathbf{S}_{b_1},\dots,\mathbf{S}_{b_K}$ satisfying $\emptyset \subset \mathbf{S}_{b_1}\subset \mathbf{S}_{b_2} \subset \dots \subset \mathbf{S}_{b_K} = \mathbf{S}_b$.
We can then similarly redefine our reward function and parameterize it by including intermediary barrier sets  $\mathbf{S}_{b_k}$
\begin{equation}
\begin{small}
    \mathcal{R}^{\text{cur}}(s,a;\mathbf{S}_{b_k}) = 
    \begin{cases} 
       \mathcal{R}'(s,a) -  M & s \in \mathbf{S}_{b_k} \\
       \mathcal{R}'(s,a)      & s \not\in  \mathbf{S}_{b_k}
   \end{cases}
  \end{small}
\end{equation}
Note that the sets $\mathbf{S}_{b_k}$ only change the reward associated with the states, not the dynamics.

Curriculum learning by evolving barrier set size is more interpretable and controllable
than the general reward weight approach since for each task $\mathcal{M}_t^{k}$, an agent learns a policy that avoids a subset of states, $\mathbf{S}_{b_k}$. In the general reward weight approach, it is unclear which states the resulting policy will never visit. A shortcoming of the barrier set size approach is that the convergence guarantee is limited to single barriers because if we have multiple barriers, we may not find a initial set $\mathbf{S}_{b_1}$ as described in Lemma~\ref{lemma:x_1}. The algorithm for the barrier set approach follows the same structure as Algorithm 1.

\prg{Analysis}\label{sec:convergence}
For both curriculum learning by reward weight and barrier set size, if the agent can successfully find an optimal policy at every intermediary task, then we
can find $\pi_t^*$ for $\mathcal{M}_t$. For the reward weight approach, we cannot prove that at every stage $k$, the optimal policy for $\mathcal{M}_t^{k}$ is guaranteed to be obtained, but we can still have the following proposition:
\begin{proposition}
For curriculum learning by reward weight, in every stage, the learned policy achieves a higher reward than the initialized policy evaluated on the final target task.
\end{proposition}
Though the reward weight approach is not guaranteed to achieve the optimal policy in every curriculum step, the policy improves with respect to the final target reward. Each curriculum step is much easier than the original direct fine-tuning problem, which increases the possibility for successful fine-tuning. For the barrier set size approach, we prove that in every stage, the optimal policy for each stage is achievable. To learn an optimal policy in each stage, finding $\mathbf{S}_{b_1}$ is key:
\begin{lemma}\label{lemma:x_1}
There exists $\mathbf{S}_{b_1}$ that divides the trajectories of $\pi^*_{s}$ and $\pi^*_{\text{relax}}$ into two homotopy classes.
\end{lemma}
We propose an approach for finding $\mathbf{S}_{b_1}$ in Algorithm 2 in Section III of the supplementary materials.
 
\begin{proposition}\label{theorem:cur}
A curriculum starting with $\mathbf{S}_{b_1}$ as described in Lemma \ref{lemma:x_1} and inflating to $\mathbf{S}_b$ with sufficiently small changes in each step, i.e., small enough 
for reinforcement learning to find trajectories that should not visit barrier states, can always learn the optimal policy $\pi_t^*$ for the final target reward.
\end{proposition}
\section{Experiments}
\label{sec:experiments}
We evaluate our approach on four axes of complexity:
(1) the size of barrier,
(2) the number of barriers,
(3) barriers in 3D environments, and
(4) barriers that are not represented by physical obstacles but by a set of `undesirable' states.

To evaluate these axes, we use various domains including navigation (Figs.~\ref{fig:navigation1}, \ref{fig:navigation2}), lunar lander (Fig.~\ref{fig:lander_fetch} Left), fetch reach (Fig.~\ref{fig:lander_fetch} Right), mujoco ant (Fig.~\ref{fig:ant}), and assistive feeding task (Fig.~\ref{fig:feeding}). We compare our approach against na\"ive fine-tuning (Fine-tune) as well as three state-of-the-art fine-tuning approaches: Progressive Neural Networks (PNN)~\cite{rusu2016progressive}, Batch Spectral Shrinkage (BSS)~\cite{chen2019catastrophic}, and $L^2$-SP~\cite{xuhong2018explicit}. We also add training on the target task from a random initialization (Random) as a reference, but we do not consider Random as a comparable baseline because it is not a transfer learning algorithm.
We evaluate all the experiments using the total number of interaction steps it takes to reach within some small distance of the desired return in the target task. We report the average number of interaction steps over in units of 1000 (lower is better). We indicate statistically significant differences ($p < 0.05$) with baselines by listing the first letter of those baselines.  We ran Navigation (barrier sizes), and Fetch Reach experiments with 5 random seeds and the rest with 10 random seeds. If more than half of the runs exceeded the maximum number of interaction steps without reaching the desired target task reward, we report that the task is unachievable with the maximum number of interaction steps. Finally, we use stochastic policies which is why our source and target policies may not be symmetrical.
Experiment details are in Section V of the supplementary materials.

\begin{figure}[h!]
    \centering
    \includegraphics[width=35mm, height=28mm]{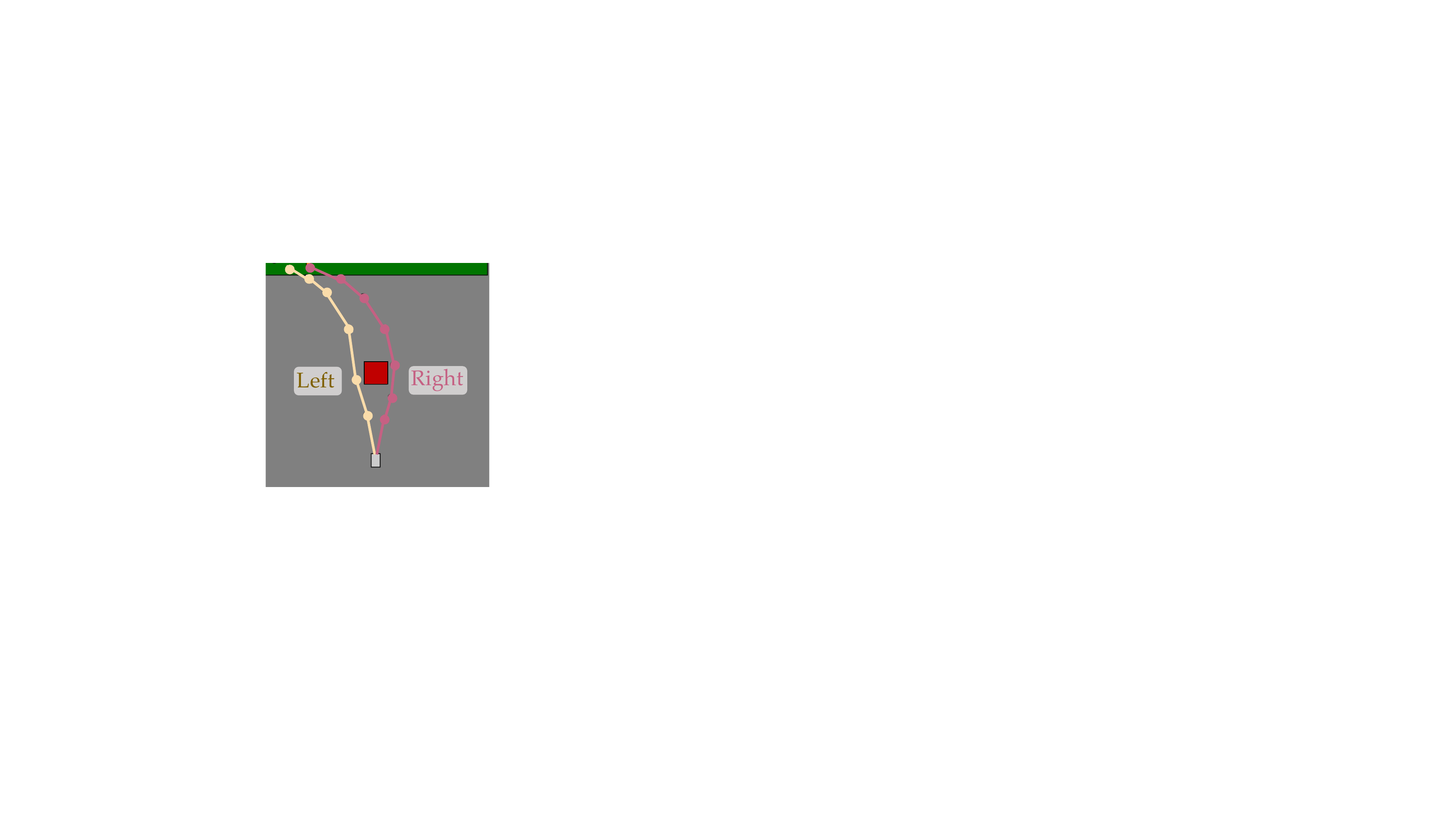}
    \caption{Navigation environment with barrier size 5.}
    \label{fig:navigation1}
    \vspace{-15pt}
\end{figure}
\begin{table}[h!]
		\centering
		\scalebox{0.8}{
    	\begin{tabular}{ c|c c c c}
    	  & \multicolumn{4}{c}{\textbf{Barrier Sizes}} \\
          & 1 & 3 & 5 & 7 \\
         \hline
        \noalign{\vskip 2pt}    
            Ours         & 117.4 $\pm$ 128.6            & 162.6 $\pm$ 70.5$^{f}$         & \textbf{102.7} $\pm$ 87.8$^{l,b}$  & \textbf{112.3} $\pm$ 111.3$^{l,b,f}$ \\
            PNN          & \textbf{92.2} $\pm$ 102           & \textbf{138.6} $\pm$ 92.1 & 159.8 $\pm$ 90.6            & 119.2 $\pm$ 125 \\
            $L^2$-SP     & 138.2 $\pm$ 61.3           & >256        & >256                        & >256    \\
            BSS          & >256          & >256         & >256                        & >256\\
            Fine-tune    & 141.1 $\pm$ 53           & >256         & 157 $\pm$ 100         & 241 $\pm$ 27.5 \\
            \noalign{\vskip 2pt}    
            \hdashline 
            \noalign{\vskip 2pt}    
            Random       & 54.6 $\pm$ 61.5   & 88.4 $\pm$ 59.4         & 145 $\pm$ 74.8          & 77.1 $\pm$ 40.6    \\
        \end{tabular}
        }
        \vspace{0.5em}
        \caption{Larger barrier sizes make fine-tuning more challenging. Our approach performs comparably with small sizes and outperforms other methods with large sizes. We only use one curriculum step, so the reward weight and the barrier set size approaches are the same.}
        \label{table:barrier_size}
\end{table}

\prg{1. Navigation}
We address the first two axes by analyzing our problem under varying barrier sizes and varying number of homotopy classes. We experiment with our running example where an agent must navigate from a fixed start position $s_i$ to the goal set $\mathbf{S}_g$ (green area). 

\noindent \emph{Varying Barrier Sizes.} We investigate how varying the size of the barrier affects the fine-tuning problem going from Right to Left. Here, we use a  one-step curriculum so the barrier set size and reward weight approaches are the same. Table~\ref{table:barrier_size} demonstrates that when barrier sizes are small (1,3), our approach is not the most sample efficient, but remains comparable to other methods. With larger barrier sizes (5, 7), we find that our method requires the least amount of training updates. \textbf{This result suggests that our approach is especially useful when barriers are large (i.e., fine-tuning is hard). When fine-tuning is easy, simpler approaches like starting from a random initialization can be used.}

\begin{figure}[h!]
    \centering
    \includegraphics[width=87mm, height=30mm]{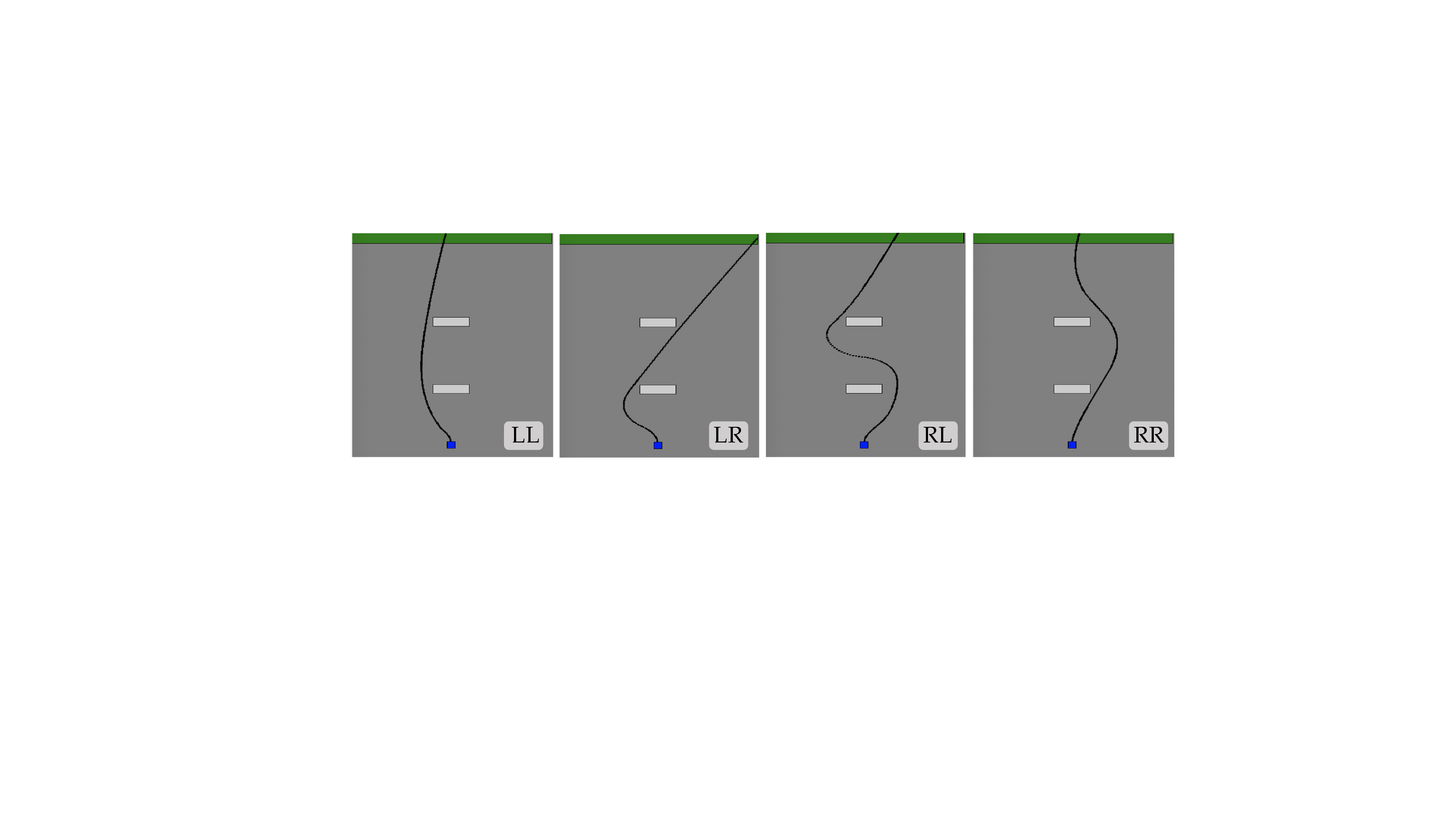}
    \caption{Navigation environment with four homotopy classes.}
    \label{fig:navigation2}
\end{figure}

\begin{table} [h!]
		\begin{center}
		\scalebox{1.}{
    	\begin{tabular}{ c|c c c }
    	  & \multicolumn{3}{c}{\textbf{Transfer Tasks}} \\
          & LL $\rightarrow$ LR & LL $\rightarrow$ RL & LL $\rightarrow$ RR \\
         \hline
        \noalign{\vskip 2pt}    
            Ours:barrier  & 88.1$\pm$3.2  & 52.1$\pm$14.2 & \textbf{48.1}$\pm$13.2$^{p,l,b,f}$  \\
            Ours:reward  & \textbf{63.2}$\pm$9.1$^{p,l,b,f}$   & \textbf{47.1}$\pm$10.9$^{p,l,b,f}$  & 56.5$\pm$9.9  \\
            PNN   & 101.9$\pm$ 37.2   & >300 & 119.2 $\pm$ 36.4          \\
            $L^2$-SP  & 130.6$\pm$ 28.6  & >300 & >300  \\
            BSS   & >300 & >300    &>300      \\
            Fine-tune   & 141.2$\pm$12.1 & >300  & >300 \\
            \noalign{\vskip 2pt}    
            \hdashline
            \noalign{\vskip 2pt}    
            Random    & 43.5$\pm$4.1       & >300 & 169.4$\pm$27.1          \\
        \end{tabular}
        }
        \end{center}
        \vspace{0.5em}
        \caption{Fine-tuning with multiple homotopy classes.}
        \label{table:navigation2}
	\vspace{-15pt}
\end{table}

\noindent \emph{Four Homotopy Classes.} We next investigate how multiple homotopy classes can affect fine-tuning. As shown in Fig.~\ref{fig:navigation2}, adding a second barrier creates four homotopy classes: \emph{LL}, \emph{LR}, \emph{RL}, and \emph{RR}. We experiment with both barrier set size and reward weight approaches and report results when using \emph{LL} as our source task in Table~\ref{table:navigation2}. Results for using \emph{LR}, \emph{RL}, and \emph{RR} as the source task are included in the supplementary materials. We can observe that the proposed Ease-In-Ease-out approach outperforms other fine-tuning methods. Having multiple barriers does not satisfy the single barrier assumption, so our reward weight approach performs better on average than the barrier set size approach. Note that in \emph{LL} $\rightarrow$ \emph{LR}, \textit{Random} performs best, which implies that the task is easy to learn from scratch and no transfer learning is needed. \textbf{We conclude that while increasing the number of barrier sets can result in a more challenging fine-tuning problem for other methods, it does not negatively affect our approach.}

\prg{2. Lunar Lander}
Before exploring 3D environments that differ significantly from the navigation environment, we conducted an experiment in Lunar Lander. The objective of the game is to land on the ground between the two flags without crashing. As shown in Fig.~\ref{fig:lander_fetch} (Left), this environment is similar to the navigation environments in that we introduce a barrier which creates two homotopy classes: Left and Right. However, the main difference is that the agent is controlled by two lateral thrusters and a main engine. 

\begin{table}[h!]
		\centering
		\scalebox{1}{
    	\begin{tabular}{ c|c c  }
    	 & \textbf{Lunar Lander}\\
          & L $\rightarrow$ R & R $\rightarrow$ L \\
         \hline
        \noalign{\vskip 2pt}    
            Ours:barrier & 80.46$\pm$46.58   & 80.23$\pm$39.76  \\
            Ours:reward  & \textbf{75.13}$\pm$34.25$^{p,b,f}$   & \textbf{38.43}$\pm$6.46$^{p,l,b,f}$  \\
            PNN   &  117.35$\pm$3.35 & 128.59$\pm$44.56\\
            $L^2$-SP  & 124.54$\pm$69.99  & 94.59$\pm$51.23 \\
            BSS   & >300          & >300          \\
            Fine-tune   & >300    & >300         \\
            \noalign{\vskip 2pt}    
            \hdashline
            \noalign{\vskip 2pt}    
            Random    & 232.32$\pm$ 48.21 & 162.92$\pm$49.54 \\
        \end{tabular}
        }
        \caption{}
        \label{table:lunar_lander}
	\caption*{Our approach outperforms baselines in the Lunar Lander domain.}
	\vspace{-12pt}
\end{table}

Results are shown in Table~\ref{table:lunar_lander}. We observe that while L$^2$-SP suffers from a large variance and PNN needs many more steps, both our reward weight approach and barrier set size approach outperforms the fine-tuning methods. The reward weight approach has a small standard deviation and performs stably. Note that Random requires large amount of interaction steps, meaning that training the landing task is originally quite difficult and needs transfer reinforcement learning. \textbf{Our approach significantly reduces the number of steps needed to learn the optimal policy in both directions.}

\prg{2. Fetch Reach} We address the third axis by evaluating our \ours\ approach on a more realistic Fetch Reach environment~\cite{brockman2016openai}. The Fetch manipulator must fine-tune across homotopy classes in $\mathbb{R}^3$. In the reaching task, the robot needs to reach either the orange or blue tables by stretching right or left respectively. The tables are separated by a wall which creates two homotopy classes, as shown in Fig.~\ref{fig:lander_fetch}. 
Our results are shown in Table~\ref{table:fetch_reach}. We find that our approach was the most efficient compared to baseline methods. One reason why the baselines did not perform well was that the wall's large size and its proximity to the robot caused it to collide often, making it particularly difficult to fine-tune across homotopy classes. We found that even training from a random initialization proved difficult. For this reason, we had to relax the barrier constraint to obtain valid Left and Right source policies.

\begin{figure}[h!]
    \centering
    \includegraphics[scale=0.62]{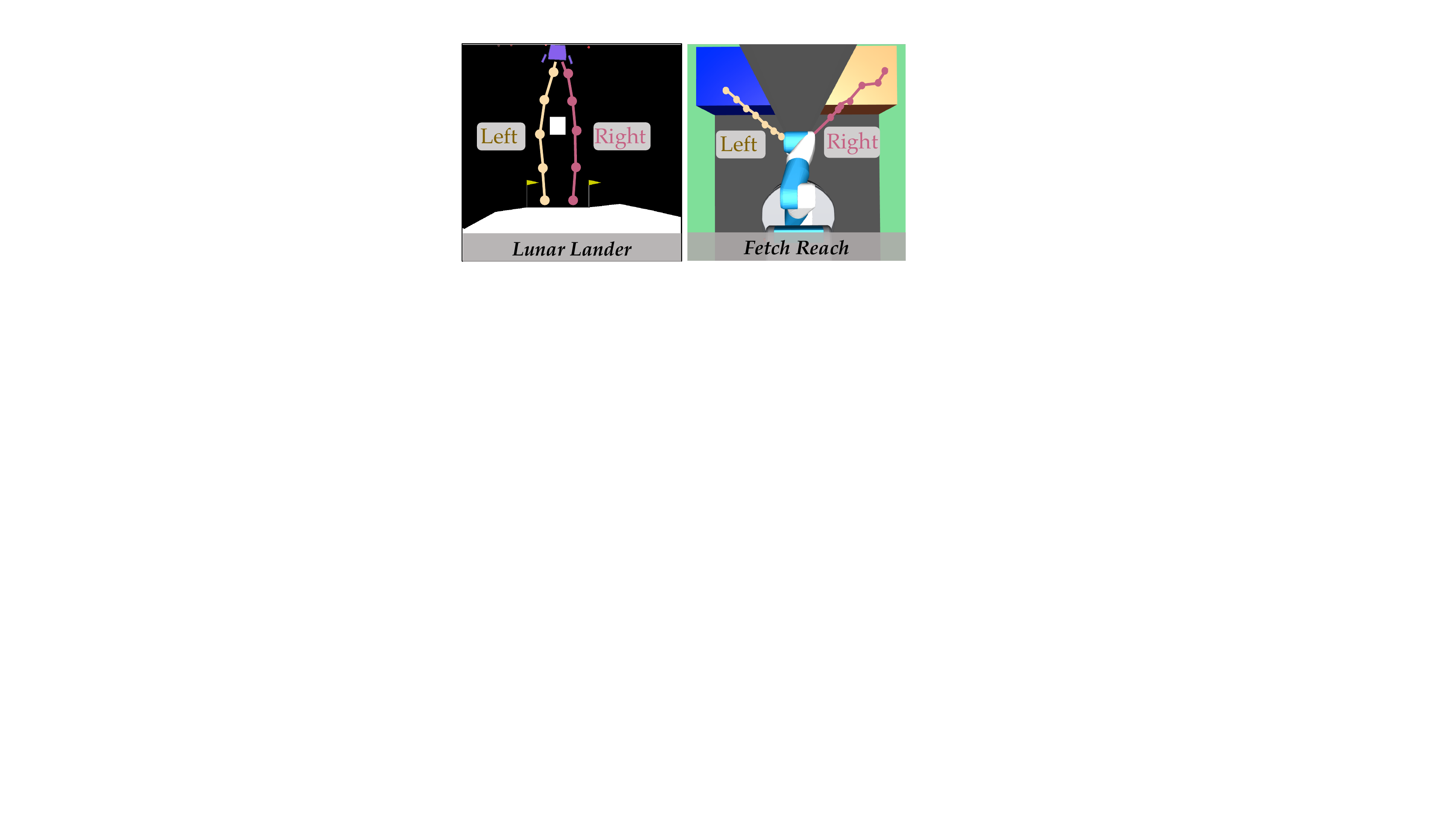}
    \caption{(Left) Lunar lander environment with two homotopy classes. (Right) Fetch reach environment. The robot must learn to reach to the right or left of the wall. }
    \label{fig:lander_fetch}
    \vspace{-15pt}
\end{figure}
\begin{table}[h!]
		\centering
		\scalebox{1.}{
    	\begin{tabular}{ c|c c }
    	 & \textbf{Fetch Reach}\\
          & L $\rightarrow$ R & R $\rightarrow$ L \\
         \hline
        \noalign{\vskip 2pt}    
            Ours:barrier  & \textbf{308.7}$\pm$167.7$^{p,b}$   & \textbf{274}$\pm$130.5$^{p,l,b,f}$ \\
            PNN   & >500           & >500    \\
            $L^2$-SP  & >500          & >500    \\
            BSS   & >500           & >500   \\
            Fine-tune & >500          & >500  \\
            \noalign{\vskip 2pt}    
            \hdashline
            \noalign{\vskip 2pt}    
            Random   & >500 & >500   \\
        \end{tabular}
        }
        \vspace{0.5em}
        \caption{Our approach overcomes challenging domains where the barrier is extremely close to the robot and collision (and negative rewards) during training is frequent. Other methods are not able to find good policies as efficiently.}
        \label{table:fetch_reach}
	\vspace{-10pt}
\end{table}

\prg{3. Mujoco Ant} Finally, we explore whether our algorithm can generalize beyond navigation-like tasks that are traditionally associated with homotopy classes. We demonstrate two examples--Mujoco Ant and Assistive Feeding--where barrier states correspond to undesirable states rather than physical objects. In the Mujoco Ant environment~\cite{todorov2012mujoco}, the barrier states correspond to a set of joint angles $\{x \in \frac{\pi}{4} \pm 0.2 \text{ rad}\}$ that the ant's upper right leg cannot move to. The boundary of the barrier states are shown by the red lines in Fig.~\ref{fig:ant}. In our source task, the ant moves while its upper right joint remains greater than $\frac{\pi}{4}+0.2 \text{ rad}$. We call this orientation \emph{Down}. Our goal is to transfer to the target task where the joint angle is less than  $\frac{\pi}{4}-0.2 \text{ rad}$, or \emph{Up}. Results are shown in Table~\ref{table:ant_feeding}. We do not evaluate the other direction, \emph{Up} $\rightarrow$ \emph{Down}, because this direction was easy for all of our baselines to begin with, including our own approach. \textbf{We find that our approach was the most successful in fine-tuning across the set of joint angle barrier states}.

\begin{figure}[h!]
    \centering
    \includegraphics[scale=0.58]{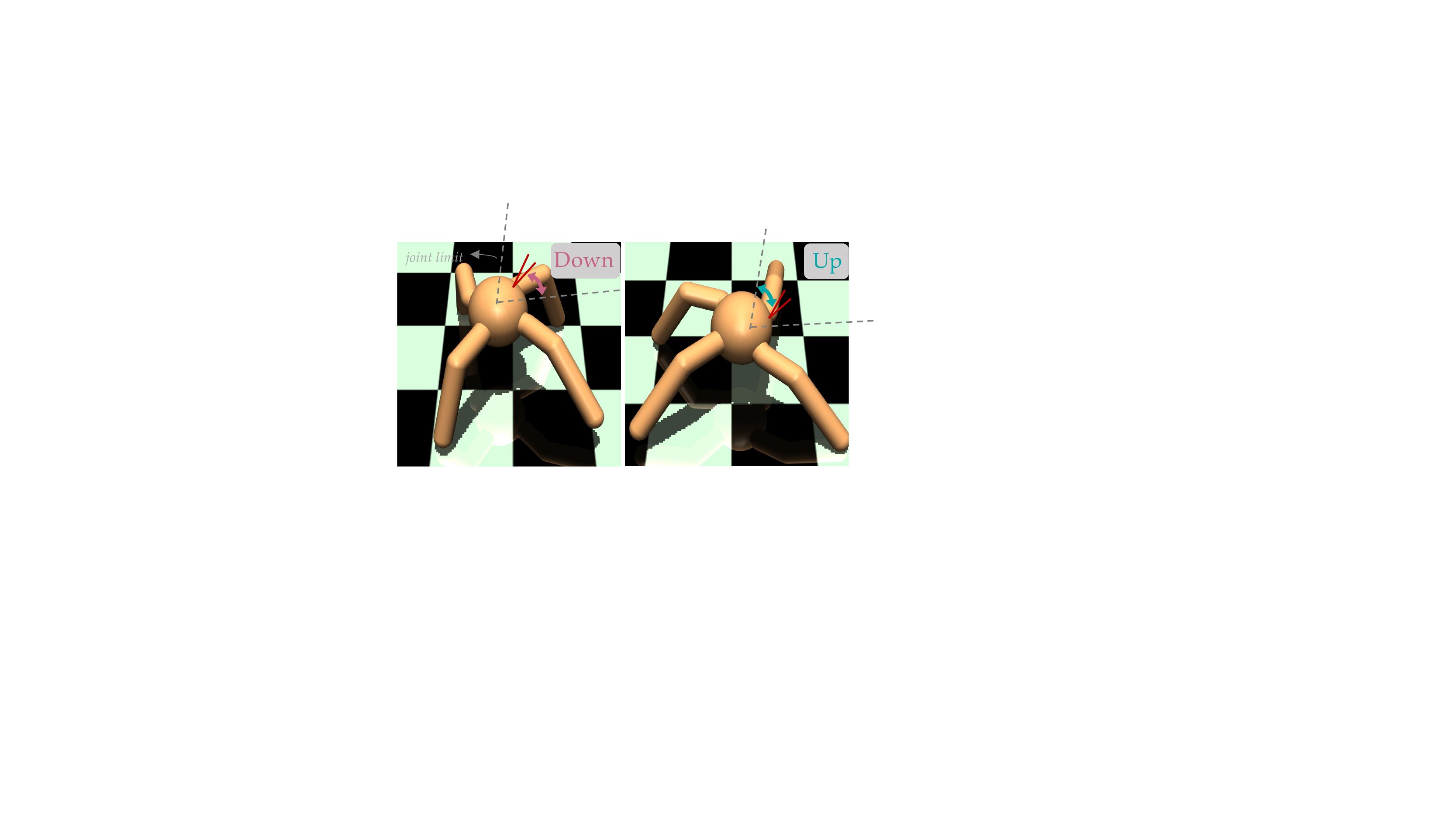}
    \caption{Mujoco Ant environment with a non-physical barrier. The red lines are the barrier states, or the joint angles the leg cannot move to. The grey dotted lines are the upper right leg's joint limits. }
    \label{fig:ant}
    \vspace{-15pt}
\end{figure}

\begin{figure}[h!]
    \centering
    \includegraphics[scale=0.55]{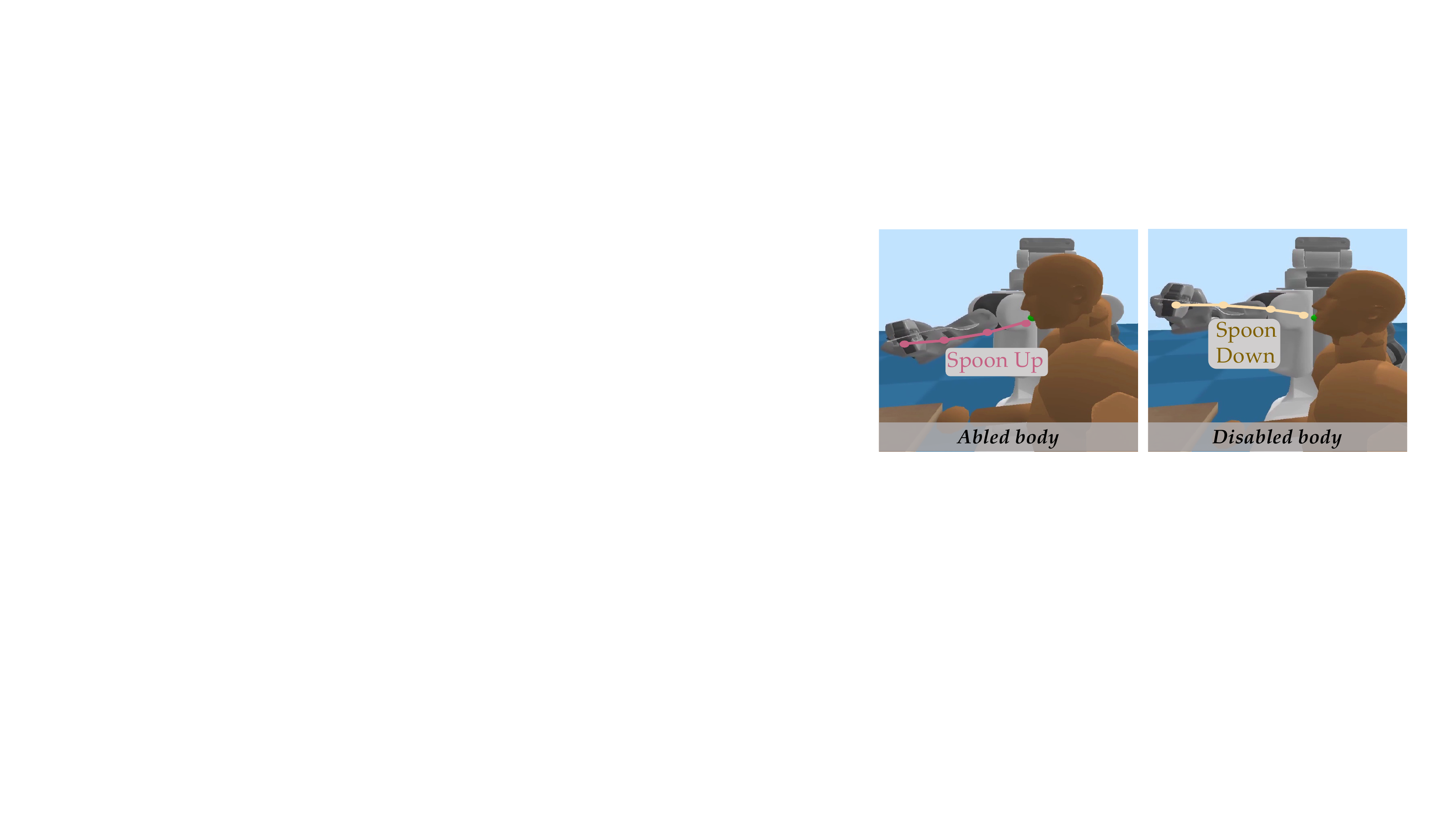}
    \caption{Assistive Feeding environment. The barrier states represent the horizontal spoon orientation. These states are undesirable for feeding because it misplaces the food in the human's mouth.}
    \label{fig:feeding}
    \vspace{-15pt}
\end{figure}

\begin{table}[h!]
		\centering
		\scalebox{1.}{
    	\begin{tabular}{ c|c c }
    	 & \textbf{Mujoco Ant} & \textbf{Assistive Feeding}\\
    	  \noalign{\vskip 2pt}    
          & Down $\rightarrow$ Up & Up $\rightarrow$ Down\\
         \hline
        \noalign{\vskip 2pt}    
            Ours  & \textbf{1420.0}$\pm$268.8$^{p,l,b,f}$ & \textbf{416}$\pm$32$^{p,l,b,f}$  \\
            PNN   &  >10000  &  >2000     \\
            $L^2$-SP  &  >10000  &  >2000    \\
            BSS   & >10000 &  >2000   \\
            Fine-tune  & 2058.5$\pm$535.2 & >2000   \\
            \noalign{\vskip 2pt}    
            \hdashline
            \noalign{\vskip 2pt}    
            Random   & 2290.4$\pm$585.8 & 494$\pm$28    \\
        \end{tabular}
        }
        \vspace{0.5em}
        \caption{Our approach works well in more general environments where barriers represent undesirable states instead of physical objects. We only use one curriculum step, so the reward weight and the barrier set size approaches are the same.}
        \label{table:ant_feeding}
	\vspace{-20pt}
\end{table}

\prg{4. Assistive Gym}
We use an assistive feeding environment~\cite{erickson2020assistivegym} to create another type of non-physical barrier in the robot's range of motion. In Fig.~\ref{fig:feeding} (right), we simulate a disabled person who cannot change her head orientation by a large amount. The goal is to feed the person using a spoon. Here, we can easily train a policy on an abled body with a normal head orientation, as in Fig.~\ref{fig:feeding} (left). However, we have limited data for the head orientation of the disabled person (the chin is pointing upwards as it is common in patients who use a head tracking device). To feed a disabled body, the spoon needs to point down, while for an abled body, the spoon needs to point up. The barrier states correspond to holding the spoon in any direction between these two directions when close to the mouth, which may `feed' the food to the user's nose or chin. This environment is an example of settings with limited data in the target environment, i.e., interacting with the disabled person. It also shows a setting with no physical barriers, and the `barrier states' correspond to the spoon orientations in between, which can be uncomfortable or even unsafe. \textbf{As shown in Table~\ref{table:ant_feeding}, Our \ours\ approach learns the new policy for the disabled person faster than training from scratch while the other fine-tuning methods fail to learn the target policy.} 
\section{Discussion}
\label{sec:discussion}
\prg{Summary} We introduce the idea of using homotopy classes to characterize the difficulty of fine-tuning between tasks with different reward functions. We propose a novel \ours\ method that first relaxes the problem and then forms a curriculum. We extend the notion of homotopy classes, which allows us to go beyond navigation environments and apply our approach on more general robotics tasks. We demonstrate that our method requires less samples on a variety of domains and tasks compared to other fine-tuning baselines.

\prg{Limitations and Future Work} 
Our work has a number of limitations. This includes the need for accessing the barrier states a priori. However, our assistive gym example is a step towards considering environments where barrier states are not as clearly defined a priori. In the future, we plan to apply our methods to other robotics domains with non-trivial homotopy classes by directly finding the homotopy classes~\cite{chen2010measuring} and then using our algorithm to fine-tune. 

\section{Acknowledgements}
We would like to thank NSF Award Number 2006388 and the DARPA Hicon-Learn project for their support. 

{\small
\bibliography{references}
\bibliographystyle{IEEEtranS}
}

\newpage

\appendix

\section{Definitions}
\begin{definition} \textbf{$W_\infty$ metric.} 
Given a metric space $X$, the $W_\infty$ distance between two distributions $\mu$ and $\nu$ on $X$, is defined as follows:
\[ W_\infty(\nu,\mu):=\inf_{\gamma\in \Gamma(\nu,\mu)} \sup_{(x, y)\in \supp(\gamma)} d(x, y),
\]
where $\Gamma(\mu, \nu)$ represents the set of couplings of $\mu$ and $\nu$, i.e., matchings of probability mass.
\end{definition}

\begin{definition} \textbf{$W_\infty$-continuity.} 
A mapping from a metric space with any distance function to a distribution space, $f: \mathbb{M} \rightarrow \mathbb{D}$ is continuous in Wasserstein-$\infty$ metric if and only if $\forall \epsilon > 0$ and $x_0\in M$, $\exists \delta > 0 $ such that if $x_0-\delta<x<x_0+\delta$, then  $W_\infty(f(x), f(x_0))< \epsilon$.
\end{definition}

\section{Proofs}
Since we are dealing with compact spaces for inputs to our functions, continuity is in general equivalent to uniform continuity. So below we interchangeably use the terms $W_\infty$ continuity and uniform $W_\infty$ continuity. We also note that uniform continuity is a weaker assumption than Lipschitz continuity, which is satisfied by neural networks with bounded parameters.

Below we show that the composition of two stochastic functions, each of which is uniformly $W_\infty$ continuous, is uniformly $W_\infty$ continuous. A stochastic function $f$ maps the input $x$ to a random $y$, and can also be viewed as a deterministic function that maps $x$ to a p.d.f.\ over $y$. The composition of $f$ and $g$ maps input $x$ first to a random intermediate output $y$ using $g$, and then $g$ is fed to $f$ to generate the random output $z$. So the composition can also be viewed as the function that maps $x$ to the p.d.f.\ of this composed output $z$. We remark that \emph{uniform} $W_\infty$ continuity is implied by $W_\infty$ continuity as long as we are working over compact spaces. We also remark that the \emph{uniform} continuity assumption can be relaxed for the inner function in the composition.

 \begin{lemma}\label{app:w_continuous_composition}
 Let $f_1(x)$, $f_2(y)$ to be two stochastic functions. Their composition, which we denote by $f_2\circ f_1$, maps $x$ to a distribution whose p.d.f.\ is defined by $\int_{y} p(y) f_2(y)dy$ where $p(y)$ is the p.d.f.\ $f_1(x)$ (note that $f_2(y)$ is a p.d.f.\ itself, so the integration results in a p.d.f.). The result of this composition is uniformly $W_\infty$ continuous as long as $f_1,f_2$ are uniformly $W_\infty$ continuous.
  \end{lemma}
\begin{proof}
Since both $f_1$ and $f_2$ are uniformly $W_\infty$ continuous functions, $\forall \epsilon > 0$, $\exists \delta_1, \delta_2 > 0$ such that
\begin{align*}
d(x, x')<\delta_1 &\implies W_\infty(f_1(x), f_1(x'))<\epsilon,\\
d(y, y')<\delta_2 &\implies W_\infty(f_2(y), f_2(y'))<\epsilon.
\end{align*}
Now fix an $\epsilon>0$. We need to show that there is a $\delta>0$ such that $d(x, x')<\delta$ implies $W_\infty(f_2\circ f_1(x), f_2\circ f_1(x))<\epsilon$. By uniform continuity of $f_2$, we first find a $\delta_2$ corresponding to $\epsilon$, and then treating $\delta_2$ as the ``$\epsilon$'' for the function $f_1$ above, we find a corresponding $\delta=\delta_1$. Our goal is to show that $d(x, x')<\delta$ implies $d(f_2\circ f_1(x), f_2\circ f_1(x))<\epsilon$.

Note that $d(x, x')<\delta_1$ implies $W_\infty(f_1(x), f_1(x'))<\delta_2$. We now take the coupling $\gamma$ from the definition of the $W_\infty$ metric for $f_1(x)$ and $f_1(x')$ which achieves the $\inf$. Note that we are making a simplifying assumption that the $\inf$ in the definition of $W_\infty$ is obtained at some particular $\gamma$; in general, this may not be the case, but tweaking our argument by taking the limit of such $\gamma$ proves the general case.

Because $W_\infty(f_1(x), f_2(x'))<\delta_2$, this means the support of $\gamma$ is only on pairs $(y, y')$ such that $d(y, y')<\delta_2$. Note that because $\gamma$'s marginals are $f_1(x), f_1(x')$, we have
\begin{align*}
	f_2\circ f_1(x)=&\int_{y, y'} \gamma(y, y')f_2(y)dy\cdot dy', \\
	f_2\circ f_1(x)=&\int_{y, y'} \gamma(y, y')f_2(y')dy\cdot dy'. \\
\end{align*}
By the triangle inequality for the $W_\infty$ metric we have
%\[
\begin{align*}
	W_\infty(f_2\circ f_1(x), f_2\circ f_1(x))\leq \\
	\int_{y, y'} \gamma(y, y') W_\infty(f_2(y), f_2(y'))dy\cdot dy'.
\end{align*}
%\]
But note that $\gamma$ is only supported on pairs $(y, y')$ such that $d(y, y')<\delta_2$. Uniform continuity of $f_2$ implies that $W_\infty(f_2(y), f_2(y'))$ for all such pairs is at most $\epsilon$. So we get
\begin{equation}
	W_\infty(f_2\circ f_1(x), f_2\circ f_1(x'))\leq \int_{y, y'} \gamma(y, y')\cdot \epsilon \cdot dy \cdot dy'=\epsilon.
\end{equation}

\end{proof}

\begin{theorem}\label{app:theorem:1}
Assume that $\pi_\theta$ is a parameterized policy for an MDP $\mathcal{M}$. If both $\pi_\theta$ and $\mathcal{M}$ are $W_\infty$-continuous, then a continuous change of policy parameters $\theta$ results in a continuous deformation of the induced random trajectory in the $W_\infty$ metric.
 However, continuous deformations of the trajectories does not ensure continuous changes of their corresponding policy parameters.
\end{theorem}
\begin{proof}
Our strategy is to prove $W_\infty$ continuity by induction and repeated applications of Lemma~\ref{app:w_continuous_composition}.

When $t=0$, $s_0$ is a deterministic constant function with respect to $\theta$, thus, it is $W_\infty$ continuous.

For $t>0$, assume by induction that we have proved $(s_0, a_0, s_1, a_1, \dots, s_{t-1})$ is $W_\infty$ continuous w.r.t.\ $\theta$. Our goal is to prove that $(s_0,a_0,\dots, s_{t-1}, a_{t-1}, s_t)$ is also $W_\infty$ continuous. We do this in two steps. First we obtain $W_\infty$ continuity of $(s_0,a_0,\dots,s_{t-1},a_{t-1})$, and then $(s_0, a_0,\dots, s_{t-1},a_{t-1}, s_t)$.

For the first step, note that
\[ (s_0,a_0,\dots,s_{t-1},a_{t-1})=f(s_0,a_0,\dots,s_{t-1}), \]
where $f=(\mathrm{id}, \pi_\theta(s_{t-1}))$ is the concatenation of the identity operator and the policy $\pi_\theta$. Since this concatenation results in a $W_\infty$ continuous function, and $\pi_\theta$ is $W_\infty$ continuous, by Lemma~\ref{app:w_continuous_composition}, we get $W_\infty$ continuity of $(s_0,a_0,\dots, s_{t-1},a_{t-1})$.

For the second step, we compose with the dynamics of $\mathcal{M}$. Namely
\[ (s_0, a_0, \dots, s_{t-1}, a_{t-1}, s_t)=g(s_0,a_0,\dots,s_{t-1},a_{t-1}), \]
where $g=(\mathrm{id}, p(s_{t-1}, a_{t-1}))$ is the concatenation of the identity operator and the MDP dynamics $p$. By a similar reasoning, this results in a $W_\infty$ continuous function of $\theta$.

For the second half of the theorem, the same trajectory may be induced by different policies, so continuous change trajectories can corresponds to quite different model parameters.
\begin{comment}
 then $p_\theta(s_{t}) = \int_{s_{t-1},a}p_\theta(s_{t-1})\pi_\theta(s_{t-1},a)p(s_{t-1},a,s_t)$. Applying the composition rule and multiplication rule, $p_\theta(s_{t})$ is also a continuous function. So all the $p_\theta(s_t)$ is continuous with respect to $\theta$.

Applying the multiplication rule, $p(\xi=\{s_0,s_1,...\})$ is also continuous.

Therefore, if both $\pi_\theta$ and $\mathcal{M}$ are $W_\infty$-continuous, then a continuous change of policy parameters $\theta$ results in a continuous deformation of the induced random trajectory in the $W_\infty$ metric.

\end{comment}

\end{proof}

\begin{lemma}\label{app:deep_rl_model} A deep RL model, which is represented by a neural network $f_\theta$ with Lipschitz continuous activation functions such as ReLU and TanH at each layer, induces a $W_\infty$ continuous policy $\pi$.
\end{lemma}
\begin{proof}
We assume the action distribution is a type of distribution $P(\mathbf{w})$ with parameters $\mathbf{w}$, such as Gaussian distribution with mean and standard deviation. $P$ can be regarded as a mapping for parameter $\mathbf{w}$ to a distribution. We assume that $P$ is $W_\infty$ continuous, which is satisfied for commonly used distributions such as the Gaussian distribution. So $\forall \epsilon>0$ and $\mathbf{w}$, $\exists \delta$ such that if $d(\mathbf{w}, \mathbf{w}_1)<\delta$, $W_\infty(P(\mathbf{w}), P(\mathbf{w}_1))<\epsilon$.

The policy of deep RL model is represented by a neural network $f_\theta$ parameterized by $\theta$, which outputs the parameters $\mathbf{w}$ of the distribution. Every layer of the neural network employs a Lipschitz continuous activation function such as ReLU and TanH. We define the number of layers in the network as $L$ and the Lipschitz constant for every layer $l \in L$ as $K_l$. We assume that the norms of the output of each layer with respect to all the states in the state space is upper bounded by a constant value $D$ or otherwise the output is infinity. Then we can derive that
\begin{equation}
\begin{aligned}
    &\lvert f_{\theta_1}(s) - f_{\theta_2}(s) \rvert \le  \Pi_{l=1}^{L} (K_l \lvert \theta_1 - \theta_2\rvert D). \\
\end{aligned}
\end{equation}
If $\lvert \theta_1 - \theta_2\rvert < \frac{\delta}{\Pi_{l=1}^{L} (K_l D)}$, then $\lvert f_{\theta_1}(s) - f_{\theta_1}(s) \rvert < \delta$, and then $W_\infty(P(f_{\theta_1}(s)), P(f_{\theta_2}(s))<\epsilon$. So the action distribution is $W_\infty$ continuous with respect to the model parameters. Then we derive the $W_\infty$ continuity of $\pi_{\theta}(s,a)$.

\begin{equation}
\begin{aligned}
    & \inf_{\gamma \in \Gamma(\pi_{\theta_1},\pi_{\theta_2})}\sup_{((s_1,a_1), (s_2,a_2))\supp(\gamma)} d((s_1,a_1), (s_2,a_2)) \\
    & =  \inf_{\gamma \in \Gamma(f_{\theta_1}(s)p(s),f_{\theta_2}(s)p(s))}\\
    &\sup_{((s_1,a_1), (s_2,a_2))\supp(\gamma)} d((s_1,a_1), (s_2,a_2)) \\
    & \le \inf_{\gamma' \in \Gamma(f_{\theta_1}(s),f_{\theta_2}(s))} \sup_{(a_1, a_2)\supp(\gamma')} d(a_1,a_2)
\end{aligned}
\end{equation}
Since the $f_{\theta}(s)$ is continuous with respect to $\theta$, then $\pi_{\theta}(s,a)$ is also $W_\infty$ continuous.
\end{proof}

\begin{remark}\label{app:theorem:2}
Fine-tuning deep RL model parameters across homotopy classes is more difficult or even infeasible in terms of sample complexity compared to fine-tuning within the same homotopy class, under the assumptions that the transition probability function and policy of $\mathcal{M}$ are $W_\infty$-continuous, learning rate is sufficiently small, and gradients are bounded.
\end{remark}

\begin{proof}[Justification for Remark~\ref{app:theorem:2}]

\textit{Large negative rewards, or barriers, correspond to high loss and create gradients pointing away from the barriers.}
Looking at Equation~(2), the large cost incurred when a trajectory $\tau$ collides with a barrier creates a large negative $R(\tau)$ term. This causes the gradient estimate $\nabla_{\theta}\hat{J}(\pi_{\theta})$ to be extremely large. 

As implied in Theorem \ref{app:theorem:1}, fine-tuning across homotopy classes implies that trajectories must intersect with the barrier at some steps during training if both the policies and the transition function are $W_\infty$-continuous. The policy induced by the deep RL model is $W_\infty$-continuous based on Lemma~\ref{app:deep_rl_model}. When intersecting with the barrier, large negative rewards create large gradients that point away from the optimal target policy weights. Thus fine-tuning across homotopy classes will always be blocked by the barriers. Instead, fine-tuning within the homotopy classes can always find a deformation process without intersecting with the barriers. 

\end{proof}

\begin{proposition1}
For curriculum learning by reward weight, in every stage, the learned policy achieves a higher reward than the initialized policy evaluated on the final target task.
\end{proposition1}
\begin{proof}
We denote $\pi_{\text{relax}}$ by $\pi^*_{\theta_0}$ and set $\alpha_0$ as $0$. For the reward weight approach, in every stage $k\ge 1$, we can write the reward function 
\begin{equation}
\begin{small}
    \mathcal{R}^{\text{cur}}(s,a;\alpha_k) = 
    \begin{cases} 
       \mathcal{R}'(s,a) - \alpha_k M & s \in \mathbf{S}_b \\
       \mathcal{R}'(s,a) & s \not\in  \mathbf{S}_b
   \end{cases}
  \end{small}
\end{equation} as a sum of two parts 
\begin{equation}
\begin{small}
    \mathcal{R}^{\text{cur}}(s,a;\alpha_{k-1}) = 
    \begin{cases} 
       \mathcal{R}'(s,a) - \alpha_{k-1} M & s \in \mathbf{S}_b \\
       \mathcal{R}'(s,a) & s \not\in  \mathbf{S}_b
   \end{cases}
  \end{small}
\end{equation} and 
\begin{equation}
\begin{small}
\begin{aligned}
    \mathcal{R}^{\text{cur}}_{\text{diff}}(s,a;\alpha_k,\alpha_{k-1}) & = 
    \begin{cases} 
       (\alpha_{k-1}-\alpha_{k}) M & s \in \mathbf{S}_b \\
       0 & s \not\in  \mathbf{S}_b.
   \end{cases} \\
   & = (\alpha_{k}-\alpha_{k-1})\mathcal{R}^{\text{bar}} \\
   & where\ \mathcal{R}^{\text{bar}} = \begin{cases} 
       - M & s \in \mathbf{S}_b \\
       0 & s \not\in  \mathbf{S}_b.
   \end{cases}
   \end{aligned}
  \end{small}
\end{equation}
$\mathcal{R}^{\text{bar}}$ can be regarded as the reward function which only penalizes the barrier states.
Then if evaluated under $\mathcal{R}^{\text{cur}}(s,a;\alpha_{k-1})$, since $\pi^*_{\theta_{k-1}}$ is trained to maximize the expected return under $\mathcal{R}^{\text{cur}}(s,a;\alpha_{k-1})$,$\pi^*_{\theta_{k-1}}$ achieves higher expected return than $\pi^*_{\theta_{k}}$. However, if evaluated under $\mathcal{R}^{\text{cur}}(s,a;\alpha_k)$, similarly, $\pi^*_{\theta_{k}}$ achieves higher expected return than $\pi^*_{\theta_{k-1}}$. Therefore, $\pi^*_{\theta_{k}}$ achieves higher expected return than $\pi^*_{\theta_{k-1}}$ if evaluated under the reward $\mathcal{R}^{\text{cur}}_{\text{diff}}(s,a;\alpha_k,\alpha_{k-1})$. Since $\alpha_{k}-\alpha_{k-1}>0$, so $\pi^*_{\theta_{k}}$ achieves higher expected return than $\pi^*_{\theta_{k-1}}$ under reward $\mathcal{R}^{\text{bar}}$, which means that $\pi^*_{\theta_{k}}$ is penalized less than $\pi^*_{\theta_{k-1}}$ by the barrier states. Now $\pi^*_{\theta_{k}}$ has achieved higher expected return than $\pi^*_{\theta_{k-1}}$ under both $\mathcal{R}^{\text{cur}}(s,a;\alpha_{k})$ and $\mathcal{R}^{\text{bar}}$, and thus also achieved higher expected return under the final target reward $\mathcal{R}_t = \mathcal{R}^{\text{cur}}(s,a;\alpha_{k})+(1-\alpha_{k})\mathcal{R}^{\text{bar}}$.
\end{proof}

\begin{lemma}\label{app:lemma:x_1}
There exists $\mathbf{S}_{b_1}$ that divides the trajectories of $\pi^*_{s}$ and $\pi^*_{\text{relax}}$ into two homotopy classes.
\end{lemma}
\begin{proof}
If there is no such $\mathbf{S}_{b_1}$, then the trajectories of $\pi^*_{\text{relax}}$ and $\pi^*_{s}$ visit no states in $\mathbf{S}_b$ and can be continuously deformed to each other without visiting any state in $\mathbf{S}_b$. So $\pi^*_{\text{relax}}$ and $\pi^*_{s}$ are in the same homotopy class with respect to $\mathbf{S}_b$. Since $\pi^*_{\text{relax}}$ optimizes the reward $R'_t$ and visits no states in $\mathbf{S}_b$, it should be the optimal policy for the target reward $R_t$. Therefore, the source and target reward are in the same homotopy class, which violates the assumption that they are in different homotopy classes. Thus, there must exist such $\mathbf{S}_{b_1}$.
\end{proof}

\begin{proposition}\label{app:theorem:cur}
A curriculum beginning with $\mathbf{S}_{b_1}$ as described in Lemma~\ref{app:lemma:x_1} is inflated to $\mathbf{S}_b$ with sufficiently small changes in each step, i.e., small enough to enable reinforcement learning algorithms to approximate optimal policies whose trajectories do not visit barrier states. This curriculum can always learn the optimal policy $\pi_t^*$ for the final target reward in the last step.
\end{proposition}
\begin{proof}
We demonstrate the case with only two homotopy classes. If there are multiple homotopy classes, then the obstacle state set contains multiple non-connected subsets, and we can design the $\mathbf{S}_{b_k}$ to form a sequence of transfer tasks to gradually transfer from the source homotopy class to the target homotopy class where two homotopy classes in each transfer task only have one connected obstacle state subset between them. Thus the analysis for two homotopy classes is sufficient as it can be applied to each transfer task to ensure the convergence. Importantly, here we assume to use reinforcement learning algorithms that can converge to optimal policies for each task in the curriculum. We prove this proposition in two steps:

\textit{1) The optimal policy $\pi_1^*$ for the first step with $\mathbf{S}_{b_1}$ is in the same homotopy class as the optimal policy $\pi_t^*$ for the final target reward }

As shown in Lemma~\ref{app:lemma:x_1}, there exists $\mathbf{S}_{b_1}$ that divides the trajectories of $\pi_s^*$ and $\pi_{\text{relax}}^*$ into two homotopy classes. So in the first stage, the barrier $\mathbf{S}_{b_1}$ is between $\pi_s^*$ and $\pi_1^*$ (assuming our reinforcement learning algorithm is able to find the optimal policy $\pi_1^*$). Since there are only two homotopy classes and $\pi_s^*$ and $\pi_1^*$ are from different homotopy classes, the trajectories of $\pi_t^*$ and $\pi_1^*$ have to be from the same homotopy class.

\textit{2) In every step $k+1,(k>=1)$, if in the previous step, $\pi_k^*$ can be learned and its trajectories are in the same homotopy class as trajectories induced by $\pi_t^*$ under $\mathbf{S}_{b_k}$, then the current step's optimal policy $\pi_{k+1}^*$ can also be learned and its trajectories are from the same homotopy class as trajectories induced by $\pi_t^*$ under $\mathbf{S}_{b_{k+1}}$}

In step $k+1$, the reinforcement learning algorithm only needs to spend very little exploration to prevent the policy from visiting $\mathbf{S}_{b_{k+1}}$ since we assume that the inflation from $\mathbf{S}_{b_{k}}$ to $\mathbf{S}_{b_{k+1}}$ is small enough to make reinforcement learning able to find trajectories not visiting states in $\mathbf{S}_{b_{k+1}}$. Since the last step's optimal policy $\pi_k^*$ is in the same homotopy class as $\pi_t^*$, $\pi_{k+1}^*$, which is fine-tuned from $\pi_k^*$, cannot be in the same homotopy class as $\pi_s^*$ due to the larger set of barrier states between $\pi_k^*$ and $\pi_s^*$. So $\pi_{k+1}^*$ can only be in the same homotopy class as $\pi_t^*$.

With the above two statements, we can derive that at every step, a policy without visiting the barrier states is achievable and it is in the same homotopy class as $\pi_t^*$. In the final step, such policy is exactly $\pi_t^*$.

\end{proof}

\section{Algorithm}

In this section, we first provide our main curriculum learning algorithm in Algorithm~\ref{app:alg:weight_cur}.

\begin{algorithm}[ht]
\KwIn{Optimal policy $\pi^*_{\text{relax}}$ for the relaxed problem $\mathcal{M}^{\text{relax}}_t$, curriculum parameters $\alpha_1,\alpha_2,...,\alpha_K$ or $\mathbf{S}_{b_1},\mathbf{S}_{b_2},...,\mathbf{S}_{b_K}$} 
 \For{$k=1$ \KwTo $K$}{
    \eIf{$k==1$}{
        Initialize the policy $\pi_{\theta_k}$ with $\pi^*_{\text{relax}}$\;
    }{
        Initialize the policy $\pi_{\theta_k}$ with $\pi_{\theta_{k-1}}$\;
    }
    Train $\pi_{\theta_k}$ in task $\mathcal{M}_t^{k}$ with model-free reinforcement learning algorithm until convergence\;
 }
 \KwOut{$\pi_{\theta_K}$ as the optimal policy for the target MDP $\mathcal{M}_t$.}
 %\caption{Curriculum Learning by Reward Weight for General Case}\label{alg:weight_cur}
 \caption{Curriculum Learning}\label{app:alg:weight_cur}

\end{algorithm}

As described in Lemma~\ref{app:lemma:x_1}, we have shown that there exists a $\mathbf{S}_{b_1}$ that divides the trajectories of $\pi^*_{s}$ and $\pi^*_{\text{relax}}$ into two homotopy classes. This is particularly needed for the Barrier Set Size approach. Lemma~\ref{app:lemma:x_1} does not provide an approach for constructing $\mathbf{S}_{b_1}$. Here we propose one approach (Algorithm~\ref{app:alg:barrier_set_size}) for finding such an $\mathbf{S}_{b_1}$.

To find a desired $\mathbf{S}_{b_1}$, we assume that we have two oracle functions available: (1) a collision checker: the function $f_1(\xi, \mathbf{S})$, which outputs \texttt{true} if $\xi$ passes through $\mathbf{S}$, and outputs 
\texttt{false} otherwise; (2) a homotopy class checker: the function $f_2(\xi_1, \xi_2, \mathbf{S})$, which outputs \texttt{true} if the obstacle set $\mathbf{S}$ divide $\xi_1$ and $\xi_2$ into two different homotopy classes, and outputs \texttt{false} otherwise. These assumptions are shown in the Input in Algorithm~\ref{app:alg:barrier_set_size}.

\begin{algorithm}[!h]
\KwIn{Optimal policy for the relaxed problem $\pi_{\text{relax}}^*$, the source optimal policy $\pi_{\text{s}}^*$, a collision checker $f_1(\xi, \mathbf{S})$, a homotopy-class checker $f_2(\xi_1, \xi_2, \mathbf{S})$, the original obstacle set $\mathbf{S}_b$, Inflating size $\epsilon$}

\tcc{Find the obstacle set divides $\pi_{\text{relax}}^*$ and  $\pi_{\text{s}}^*$}

Derive the trajectory $\xi_{\text{relax}}$ induced by $\pi_{\text{relax}}^*$;

Derive the trajectory $\xi_{\text{s}}$ induced by $\pi_{\text{s}}^*$;

Assign $\mathbf{S}_{b_1} \leftarrow \mathbf{S}_b$;

\While{true}
{Cut $\mathbf{S}_{b_1}$ into two halves $\mathbf{S}_{h_1}$ and $\mathbf{S}_{h_2}$;

\uIf{$f_1(\xi_{\text{relax}}, \mathbf{S}_{h_1})$}
{Assign $\mathbf{S}_{b_1} \leftarrow \mathbf{S}_{h_1}$;

\Continue;}
\uElseIf{$f_2(\xi_{\text{relax}}, \xi_{\text{s}}, \mathbf{S}_{h_1})$}
{Assign $\mathbf{S}_{b_1} \leftarrow \mathbf{S}_{h_1}$;

\Break;}
\uElseIf{$f_1(\xi_{\text{relax}}, \mathbf{S}_{h_2})$}
{Assign $\mathbf{S}_{b_1} \leftarrow \mathbf{S}_{h_2}$;

\Continue;}
\Else
{Assign $\mathbf{S}_{b_1} \leftarrow \mathbf{S}_{h_2}$;

\Break;}
}

\tcc{Inflate the obstacle set to collide with the trajectories of $\pi_{\text{relax}}^*$}

\While{ $\neg f_1(\xi_{\text{relax}}, \mathbf{S}_{b_1})$}
{
Find a subset $\mathbf{S}'_b$ of $\mathbf{S}_b$ where $\mathbf{S}'_b$ is connected to $\mathbf{S}_{b_1}$ and $\mathbf{S}'_b \setminus \mathbf{S}_{b_1} \ne \emptyset$;

Assign $\mathbf{S}_{b_1} \leftarrow \mathbf{S}_{b_1} \bigcup \mathbf{S}'_b$;
}

  \KwOut{ First obstacle set in the curriculum $\mathbf{S}_{b_1}$.}
  \caption{Finding $\mathbf{S}_{b_1}$ for Curriculum Learning by Barrier Set Size}\label{app:alg:barrier_set_size}
\end{algorithm}

First, we observe that the trajectory $\xi_{\text{relax}}$ induced by the optimal relaxed policy $\pi_{\text{relax}}^*$ passes through the original obstacle $\mathbf{S}_b$, so there exists a subset of $\mathbf{S}_b$ causing the source optimal trajectory $\xi_{\text{s}}^*$ not being able to continuously deform to $\xi_{\text{relax}}$. This means that there exists a subset of $\mathbf{S}_b$ that divides $\pi_{\text{s}}^*$ and $\pi_{\text{relax}}^*$ into two different homotopy classes. 

Based on this idea, we first construct $\mathbf{S}_{b_1}$ which divides $\pi_{\text{s}}^*$ and $\pi_{\text{relax}}^*$ into two homotopy classes (see line 3-19). We first assign the original $\mathbf{S}_{b}$ to $\mathbf{S}_{b_1}$ (Line 3). We then cut $\mathbf{S}_{b_1}$ into two halves, $S_{h_1}$ and $S_{h_2}$ (Line 5) and update it to one of the two halves based on rules introduced below (see lines 6-18) until we reach the desirable starting barrier $\mathbf{S}_{b_1}$.

On lines 6-8 and 12-14, we use $f_1$ to check whether $\xi_{\text{relax}}$ passes through each half. For each half, if $\xi_{\text{relax}}$ passes through the half (e.g. imagine $\xi_{\text{relax}}$ passes through $S_{h_1}$), the half (in this example $S_{h_1}$) will have a strictly smaller subset dividing $\pi_{\text{s}}^*$ and $\pi_{\text{relax}}^*$ into two different homotopy classes. So we choose to further cut that half ($S_{h_1}$), while ignoring the other half ($S_{h_2}$) (see lines 6-8). 
If $\xi_{\text{relax}}$ does not pass through the half ($S_{h_1}$), we use $f_2$ to check whether this half ($S_{h_1}$) divides $\xi_{\text{relax}}$ and $\xi_{\text{s}}$ into two different homotopy classes. 
If $f_2$ return \texttt{true}, we have found a $\mathbf{S}_{b_1}$ (i.e. $S_{h_1}$) that satisfies Lemma~\ref{app:lemma:x_1} (see lines 9-11). Otherwise, since this set (the half $S_{h_1}$ in our example) does not collide with $\xi_{\text{relax}}$ and does not divide $\xi_{\text{relax}}$ and $\xi_{\text{s}}$ into two different homotopy classes, any of its subsets will not divide $\xi_{\text{relax}}$ and $\xi_{\text{s}}$ into different homotopy classes, and thus we can discard the half and move forward with the other half ($S_{h_2}$) (see lines 12-18). 

The only situation that can make the algorithm not converge is that at one step, both halves do not collide with $\xi_{\text{relax}}$ and do not divide $\xi_{\text{relax}}$ and $\xi_{\text{s}}$ into different homotopy classes. However, since $\xi_{\text{relax}}$ passes through the obstacle set at each step, the above situation can never happen. So the algorithm will converge and the convergence rate can roughly be estimated by $O(\log(|\mathbf{S}_{b}|))$.

After finding a $\mathbf{S}_{b_1}$ that satisfies Lemma~\ref{app:lemma:x_1}, we gradually inflate $\mathbf{S}_{b_1}$ until is passes through $\xi_{\text{relax}}$ (see lines 20-23), because training with such $\mathbf{S}_{b_1}$ can actually push the policy towards the target homotopy class. The illustration of the algorithm is shown in Fig.~\ref{app:fig:sb1_algorithm}.

\begin{figure*}[ht]
\centering
\subfigure[Source Optimal trajectory]{\includegraphics[width=.25\textwidth]{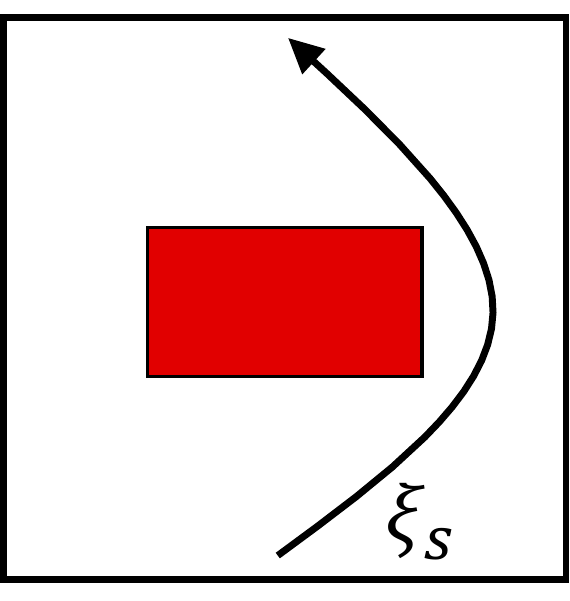}\label{app:fig:sb1_source}}
\hfil
\subfigure[Target Optimal trajectory]{\includegraphics[width=.25\textwidth]{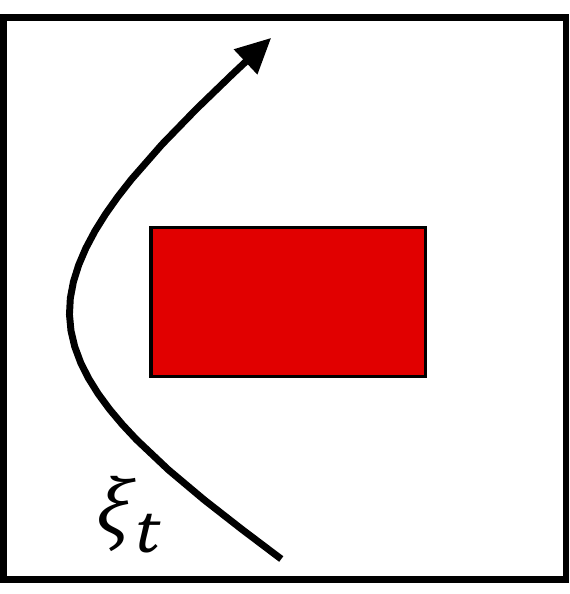}\label{app:fig:sb1_target}}
\hfil
\subfigure[Relaxed optimal trajectory]{\includegraphics[width=.25\textwidth]{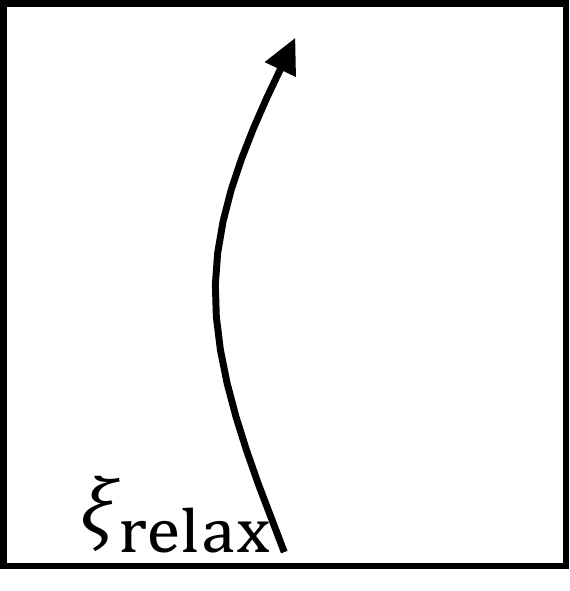}\label{app:fig:sb1_relax}}
\subfigure[Example of cutting the obstacles into halves and selecting the half colliding with $\xi_{\text{relax}}$. We then continually divide the obstacle so $\xi_s$ and $\xi_{\text{relax}}$ are eventually in different homotopy classes (See lines 3-19 of Algorithm 2)]{\includegraphics[width=0.9\textwidth]{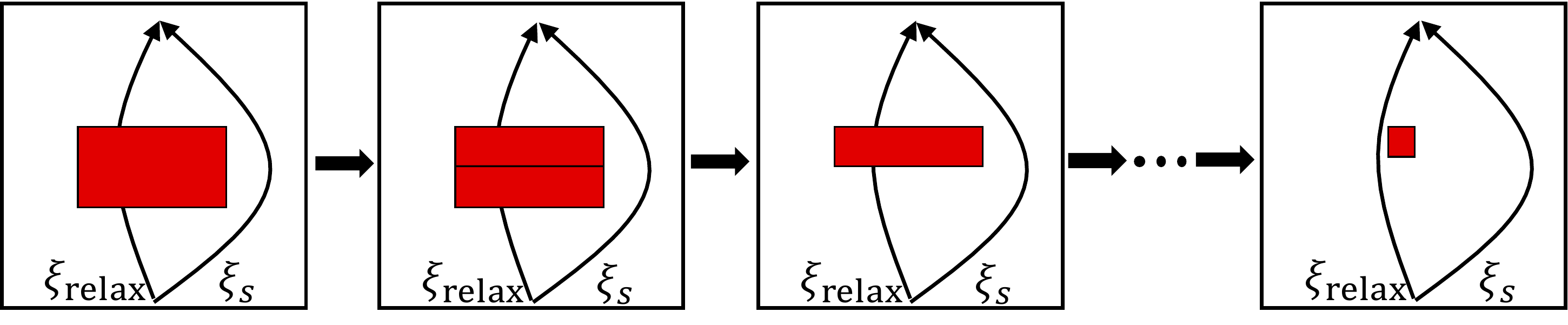}\label{app:fig:sb1_half1}}
\subfigure[Another case of cutting (See lines 3-19 of Algorithm 2)]{\includegraphics[width=0.9\textwidth]{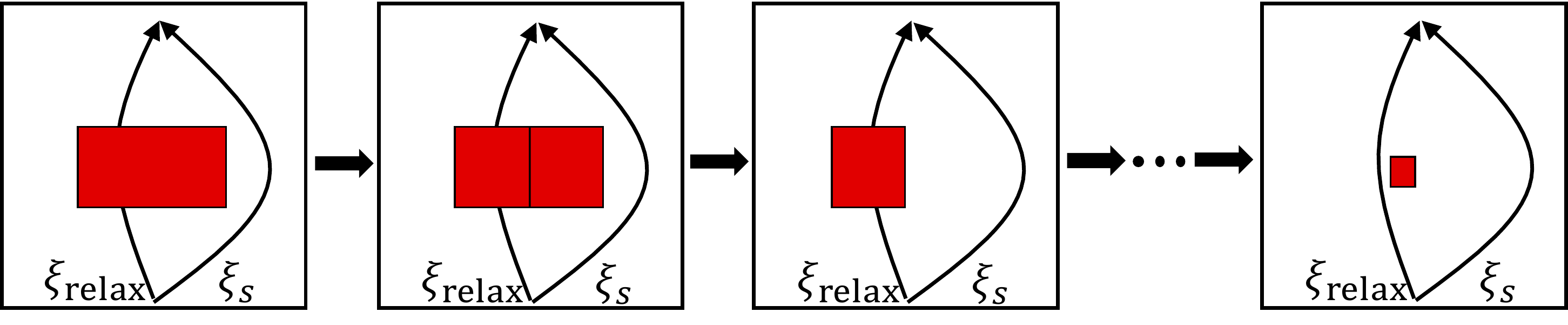}\label{app:fig:sb1_half2}}
\subfigure[Gradually inflating the obstacle and reinforcement learning with the new obstacle pushes the trajectory toward the target homotopy class.  (See lines 20-23 of Algorithm 2)]{\includegraphics[width=0.9\textwidth]{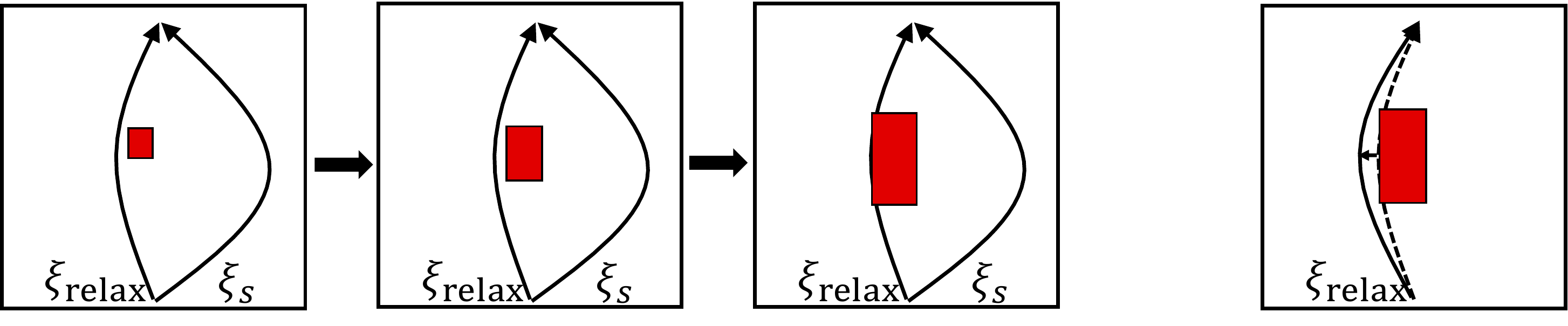}\label{app:fig:sb1_inflate}}
\caption{Illustration of the Algorithm to find the first obstacle set $\mathbf{S}_{b_1}$ for the barrier set size approach.}\label{app:fig:sb1_algorithm}
\end{figure*}

\section{Running Example: Surface Plots}
The car's goal is to navigate to the goal without colliding with the barrier. The car controls a 1-dimensional steering action and observes its 2-dimensional position. We train a linear policy using the REINFORCE algorithm with policy parameters $\theta \in \mathbb{R}^2$. Our learning rate is 1e-3 and discount factor is 0.99. Code is adapted from OpenAI's Vanilla Policy Gradient implementation \href{https://spinningup.openai.com/en/latest/_modules/spinup/algos/pytorch/vpg/vpg.html}{here}.

When generating our surface plots, the REINFORCE loss we use is $l(\theta) = \mathbb{E}_{\pi}[G_t \cdot \ln\pi_{\theta}(a_t|s_t)].$ We discretize our parameter space $\theta \in \mathbb{R}^2$ into buckets of 0.1 in between -1 and 1.3. For each of these policy parameters, we sample 10 loss values and average them to create the surface plots. Loss values are generated over 500 epochs with 100 environment steps.

\section{Experiments}
\subsection{Compared Methods}
\begin{itemize}
    \item \prg{Progressive Neural Networks} In this approach, each task is initialized as a separate neural network. Knowledge is transferred to a new task by connecting the hidden layers of previously learned tasks ~\cite{rusu2016progressive}.
    \item \prg{Batch Spectral Shrinkage} BSS prevents negative transfer to new domains by penalizing the smallest $k$ singular values of features in the source domain ~\cite{chen2019catastrophic}. The authors show that preventing negative transfer improves transfer learning results. In practice, we often found that BSS was numerically unstable, causing training to fail early. 
    \item \prg{$L^2$-SP} This approach prevents the transfer learning model from overfitting on the target domain by penalizing the squared Euclidean distance to the source model parameters ~\cite{xuhong2018explicit}. The authors show that this simple regularizer improves transfer learning results. 
\end{itemize}

\begin{table}[h!]
		\centering
		\scalebox{0.83}{
    	\begin{tabular}{ c | c c | c  }
    	 & \multicolumn{2}{c |}{\textbf{Lunar Lander}} & \multicolumn{1}{c}{\textbf{Nav1: Barrier Sizes}}\\
          & L $\rightarrow$ R & R $\rightarrow$ L & 7 \\
         \hline
        \noalign{\vskip 2pt}    
            Ours:barrier & 80.46$\pm$46.58   & 80.23$\pm$39.76 & 112.3 $\pm$ 111.3  \\
            Ours:reward  & \textbf{75.13}$\pm$34.25   & \textbf{38.43}$\pm$6.46 & "" \\
            PNN   &  117.35$\pm$3.35 & 128.59$\pm$44.56 & 119.2 $\pm$ 125  \\
            $L^2$-SP  & 124.54$\pm$69.99  & 94.59$\pm$51.23 & >256 \\
            BSS   & >300          & >300     & >256     \\
            Fine-tune   & >300    & >300   &   241 $\pm$ 27.5   \\
           \hline
            \noalign{\vskip 2pt}    
            Deeper Network & 119.32$\pm$43.28  & 88.23$\pm$45.41 & \textbf{69.2} $\pm$ 98.1\\
            Dropout& >300    & >300 & >256 \\
            Entropy Bonus & 113.29$\pm$40.12  & 97.45$\pm$39.21 & >256 \\
            \noalign{\vskip 2pt}    
            \hdashline
            \noalign{\vskip 2pt}    
            Random    & 232.32$\pm$ 48.21 & 162.92$\pm$49.54  & 77.1 $\pm$ 40.6 \\
        \end{tabular}
        }
        \caption{Comparing the use of deeper networks, dropout, and larger entropy bonuses against our approach using 5 random seeds. Our approach is able to beat all baselines in Lunar Lander and performs comparably in the Navigation. Using deeper networks works surprisingly well in Navigation and can be used in conjunction with our approach to obtain better results. }
        \label{app:table:cobb}
	\caption*{}
\end{table}
\prg{Comparison to Cobb et al.'s work~\cite{cobbe2019quantifying}} We used the Lunar Lander and the Navigation (barrier size 7) environments to compare against Cobb et al.'s work. Our experiments include using deeper neural networks, increasing the entropy bonus for exploration, and adding dropout. In all of our experiments, we used PPO to train policies using 5 random seeds. To create deeper neural networks, we doubled the number of layers in our policy and value function networks from 2 to 4, each with 64 neurons. We increased the entropy bonus from the default value of 0.01 to 0.05, which was also the optimal value in Cobb et al. To implement dropout we used a dropout rate of $p=0.5$. 

As shown in Table~\ref{app:table:cobb}, adding deeper networks, dropout, and entropy bonuses were less efficient than our approach in the Lunar Lander setting. We find that in the Navigation setting, using deeper networks performs extremely well. These methods can be used in conjunction with our approach to obtain better results.

\subsection{Implementation Details}
We introduce environment details, the main design of the reward and curriculum, training details, and how to compute the number of interaction steps for evaluation.

\prg{Navigation1: Barrier sizes} 

\textit{Datasets \& Environment.}
In this task, the car must navigate to the goal without colliding with barriers. Actions are steering angles; we use a PID controller to maintain a constant velocity. Observations include the car's position, velocity, heading, angular velocity and acceleration. The reward function is a linear combination of the car's distance to the goal, its homotopy class reward, and a barrier collision penalty. The homotopy reward is defined by the car's heading and determines how far left or right the car swerves. We anneal the distance to the goal as well as the homotopy class reward over time such that the homotopy class reward is initially dominant, but later the distance to the goal takes over. We penalize barrier collisions with a -1000 reward. The expected target return $\eta_t^*$ is calculated by training from a random initialization and averaging returns over three runs. In Table 1 of the main paper, we report the performance of our algorithm by recording the average number of interaction steps it takes to reach within $\pm 2$ of the the expected target return, shown in Table~\ref{app:table:target_returns}.

\begin{table*} [h!]
		\begin{center}
		\scalebox{1}{
    	\begin{tabular}{ c c c c|c|c|c|c}
    	 \multicolumn{4}{c|}{\textbf{Nav1: Barrier Sizes}} & \textbf{Nav2} & \textbf{Fetch Reach} & \textbf{Assistive Feeding} & \textbf{Lunar Lander}\\
         1 & 3 & 5 & 7  & &  \\
         \hline
        \noalign{\vskip 2pt}    
         61.89$\pm$2 & 62.8$\pm$2 & 60.75$\pm$2 & 53.74$\pm$2 & >3000 & >10  & 80 $\pm$ 5 & >250
        \end{tabular}
        }
        \end{center}
        \vspace{0.5em}
        \caption{Summary of the target reward thresholds we used in our evaluations. In the main paper, we report the average number of interaction steps it takes to reach these reward thresholds in the target task.}
        \label{app:table:target_returns}
\end{table*}

\textit{Models \& Algorithms.}  For all barrier sizes except 7 we use a single-step curriculum where we directly reintroduce the barrier after removing it. Note that this method is equivalent to both curriculum learning by reward weight and curriculum learning by barrier set size methods. In the reward weight approach, $\alpha$ is set to 0 and then back to 1; in the barrier set size approach, the barrier set $\mathbf{S}_b$ is set to the empty set and then back to $\mathbf{S}_b$. For barrier size 7, we use a barrier set approach where the barrier sizes we used are as follows: [4, 7].

\textit{Experimental Results.} We train policies with PPO using the Stable Baselines implementation~\cite{stable-baselines}. We train for 2000 timesteps with 128 environment steps per update (256,000 total steps). Training a policy takes around 30 minutes. We evaluate our policy every 128 steps. We use the default hyperparameters in the Stable Baselines implementation which can be found \href{https://stable-baselines.readthedocs.io/en/master/modules/ppo2.html#stable_baselines.ppo2.PPO2}{here}. In Table 1 of the main paper, we average results over 5 runs. To generate runs for our method, we sample policies from the penultimate curriculum stage instead of source policies.

\prg{Navigation2: Four Homotopy Classes.} 

\textit{Datasets \& Environment.} We use the same action space, dynamics, and observations as the Barrier Sizes experiment. However, in this experiment, we change the number of barriers and the reward function. As shown in Figure 5 in the main text, we now have two barriers with length $9$ and width $4$ in this environment. Barrier 1 is between the middle and the bottom and barrier 2 is between middle and the top. We design four homotopy classes in the form of A-B, where $A$ can be Left or Right and $B$ can be Left or Right. $A$ denotes whether the car is passing barrier 1's left or right side. $B$ means the same for barrier 2. For example, Left-Left means pass barrier 1 on the left and also passing barrier 2 on the left.

We punish the car with $-1000$ reward for each collision with barriers. If the car pass the correct side of the barrier 1 or 2, the car receives a $+500$ reward. This reward ensures that the trajectory falls in the correct homotopy class. If the car reaches the goal, the car will receive $+2000$ reward. We also add a potential field based reward to encourage the car to reach the goal more quickly. Since the potential field reward is low compared to the reward when reaching the goal introduced above, if the car collides with barriers, fails to reach the target, or does not pass the two barriers on the correct side, it will receive a reward lower than $+3000$. We thus define a trajectory with return higher than $+3000$ as a successful trajectory in the target homotopy class. We manually check each trajectory higher than $+3000$ and further discard corner cases where the trajectory is higher than $+3000$ but it is not a expected trajectory in the homotopy class. We compute the number of interaction steps based on the first time the policy achieves $>3000$ reward for baseline methods. For each curriculum, we use $>3000$ reward to decide whether the current curriculum converges.

\textit{Models \& Algorithms.} 
For the barrier set size curriculum, we have different ways to inflate the barrier state set since we have two barriers now. We try both strategies of adding the top barrier back first and then add the bottom back or add the bottom back first and then add the top back. We take the better strategy between the two. We have also tried simultaneously adding two barriers back but we found that it always needs more interaction steps to fine-tune. Since we have a rectanglular barrier, we gradually increase the barrier length from $0$ to $9$. So we set the barrier set by as a series of factors $(\beta_1, \gamma_1),\cdots,(\beta_K, \gamma_K)$, where each $\mathbf{S}_{b_{i}}$ corresponds to barrier 2 size $9\beta_i$ and barrier 1 size $9\gamma_i$. Since we add one barrier back and the next, there are two situations: $\beta_{i}|_{i=1}^{k_1}=0<\cdots<\beta_{K}=1$ and $0=\gamma_1<\cdots<\gamma_{k_1}=1=\gamma_{i}|_{i\ge k_1}$ or $\gamma_{i}|_{i=1}^{k_2}=0<\cdots<\gamma_{K}=1$ and $0=\beta_1<\cdots<\beta_{k_2}=1=\beta_{i}|_{i\ge k_2}$. To tune the hyper-parameters, we first divide the range $[0,1]$ into several ranges. Then we assess the difficulty of fine-tuning for each step. If it is too difficult, we further divide it or otherwise we preserve it. 

The exact values for $(\beta, \gamma)$ sequence for adding the barrier 2 first is:
\begin{align*}
    &[0.2,0],[0.5,0],[1,0],[3,0],[6,0],[9,0],\\
    &[9,0.2],[9,0.5],[9,1.],[9,3.], [9,6.],[9,8.],[9,8.25],\\
    &[9,8.75],[9,9] 
\end{align*}

The exact values for $(\beta, \gamma)$ sequence for adding the barrier 1 first is:
\begin{align*}
&[0,0.2],[0,0.5],[0,1.],[0,3.],[0,6.],[0,9.],[0.01,9],\\
&[0.05,9],[0.1,9],[0.2,9.],[0.5,9.],[1.,9.],[3.,9.]\\
&[6.,9.],[9,9]
\end{align*}

For reward weight approach, the $\alpha$ sequence is that
\begin{align*}
    [0.001,0.003,0.01,0.03,0.1,0.3,1.0]
\end{align*}

\begin{table*} [h!]
		\begin{center}
		\resizebox{1.0\textwidth}{!}{
    	\begin{tabular}{ c|c c c c c c c c c}
    	  & \multicolumn{9}{c}{\textbf{Transfer Tasks}} \\
          & RR $\rightarrow$ LL & RR $\rightarrow$ LR & RR $\rightarrow$ RL & LR $\rightarrow$ LL & LR $\rightarrow$ RL & LR $\rightarrow$ RR & RL $\rightarrow$ LL & RL $\rightarrow$ LR & RL $\rightarrow$ RR\\
         \hline
        \noalign{\vskip 2pt}    
            Ours:barrier  & 38.4$\pm$7.3 & 229.8$\pm$64.5 & \textbf{48.5}$\pm$10.8$^{p,l,b,f}$  & \textbf{17.9}$\pm$2.1$^{p,b,f}$ & 222.1$\pm$60.8 & 99.4$\pm$2.6 & 74.9$\pm$3.4 & 190.8$\pm$49.6$^{l,b,f}$ & \textbf{42.0}$\pm$11.0$^{p,b}$ \\
            Ours:reward  & \textbf{11.1}$\pm$4.3$^{p,l,b,f}$ & \textbf{173.5}$\pm$6.7$^{p,l,b,f}$ & 108.2$\pm$17.6 & 18.1$\pm$0.9 & \textbf{203.9}$\pm$78.0$^{p,l,b,f}$ & \textbf{49.9}$\pm$6.6$^{p,l,b,f}$ & \textbf{47.9}$\pm$4.8$^{p,l,b,f}$ & 225.9$\pm$33.6 & 49.2$\pm$9.2  \\
            PNN & 55.2$\pm$28.1 & >300 & 131.9$\pm$17.1 & 36.2$\pm$11.9 & >300 & >300 & 85.6$\pm$7.8 & \textbf{188.9}$\pm$6.7 & 99.7$\pm$35.3\\
            $L^2$-SP  & 29.2$\pm$14.6 & >300 & >300 & 22.5$\pm$5.3 & >300 & 97.4$\pm$37.9 & 79.3$\pm$5.9 & >300 & 45.5$\pm$24.2\\
            BSS   & >300 & >300 & >300 & >300 & >300 & >300 & >300 & >300 & >300 \\
            Fine-tune   & 39.6$\pm$11.6 & >300 & >300 & 31.3$\pm$9.8 & >300 & >300 & 59.6$\pm$11.1 & >300 & 52.3$\pm$12.1 \\
        \noalign{\vskip 2pt}    
        \hdashline
        \noalign{\vskip 2pt}    
            Random & 68.2$\pm$10.5 & 43.5$\pm$4.1 & >300 & 68.2$\pm$10.5 & >300 & 169.4$\pm$27.1  & 68.2$\pm$10.5 & 43.5$\pm$4.1 & 169.4$\pm$27.1  \\
        \end{tabular}
        }
        \end{center}
        \vspace{0.5em}
        \caption{Fine-tuning to multiple homotopy classes works well with our approach. In certain tasks, the target is already easily trainable from a random initialization. If this knowledge is known a-priori, there is no need to use our approach to transfer to the target task.}
        \label{app:table:navigation2}
\end{table*}

\textit{Experimental Results.} Similar to the Barrier Sizes experiment, we train policies with PPO using the Stable Baselines implementation and use their default hyperparameters, found \href{https://stable-baselines.readthedocs.io/en/master/modules/ppo2.html#stable_baselines.ppo2.PPO2}{here}. We train for total of $300,000$ environment steps with 128 environment steps per update. Training a policy takes around 60 minutes. We evaluate our policy every 1024 environment steps. In Table 2 of the main text as well as Table~\ref{app:table:navigation2} in the supplementary materials, we average results over $10$ runs. 

In the main text, we reported results for \emph{LL} $\rightarrow *$. The results for the remaining nine tasks are shown in Table \ref{app:table:navigation2}. We can observe that the our fine-tuning approach with barrier set size or reward weight outperforms all the other methods in all tasks except \emph{RL}$\rightarrow$ \emph{LR}. In \emph{RL}$\rightarrow$ \emph{LR}, the performance of our approach is still comparable with other methods.

\begin{figure*}[ht]
    \centering
    \subfigure[LL]{\includegraphics[width=.23\textwidth]{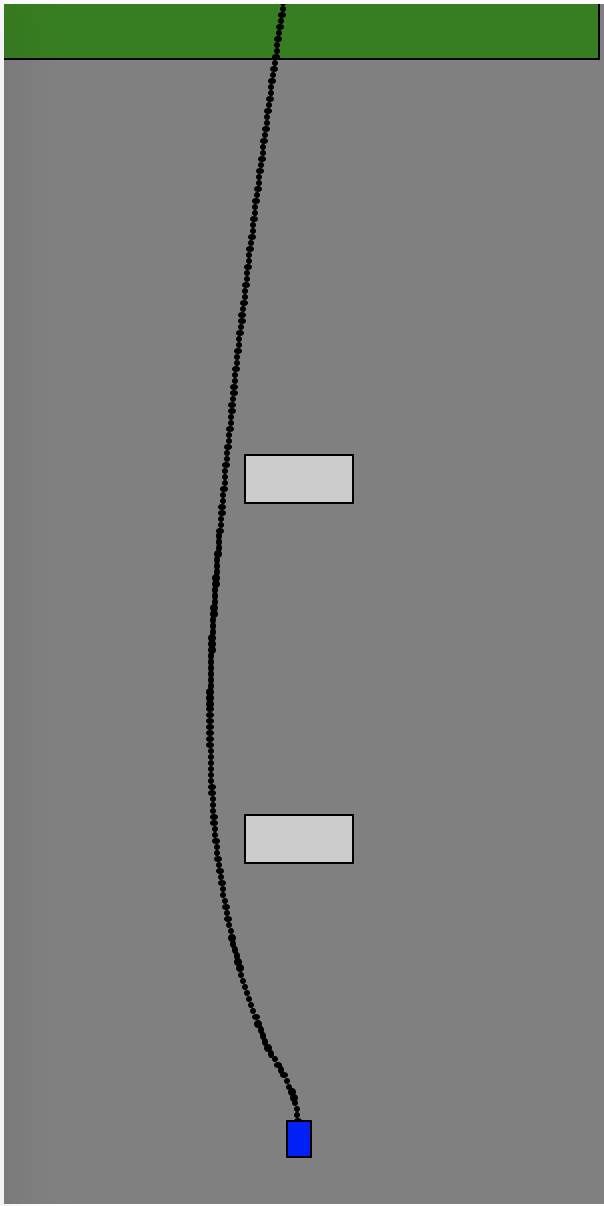}}
    \subfigure[LR]{\includegraphics[width=.23\textwidth]{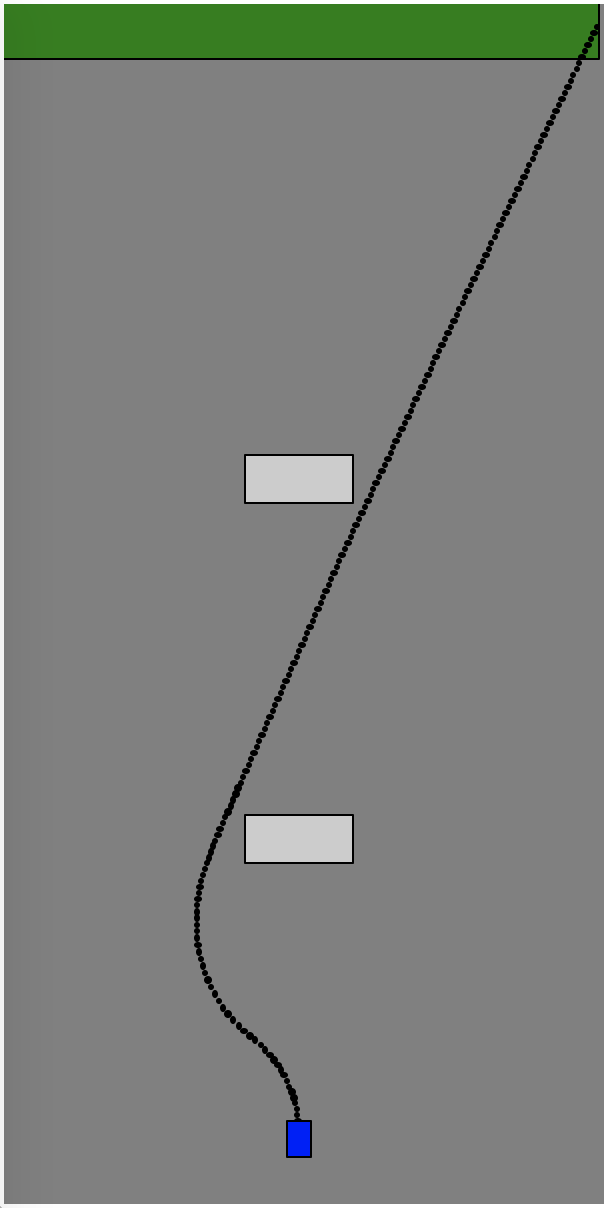}}
    \subfigure[RL]{\includegraphics[width=.23\textwidth]{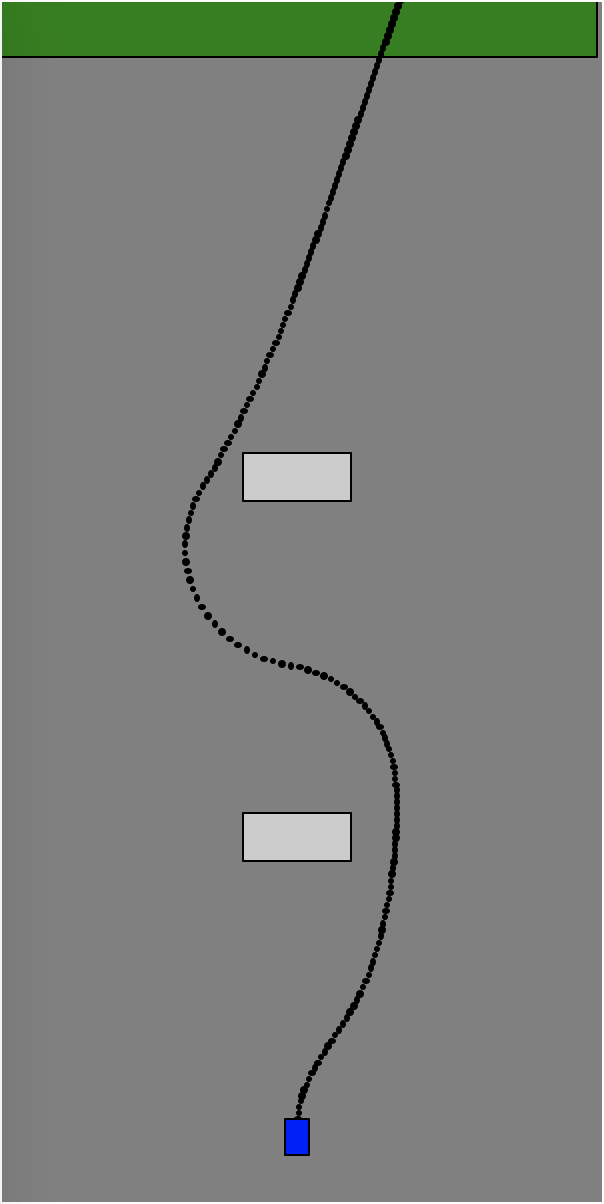}}
    \subfigure[RR]{\includegraphics[width=.23\textwidth]{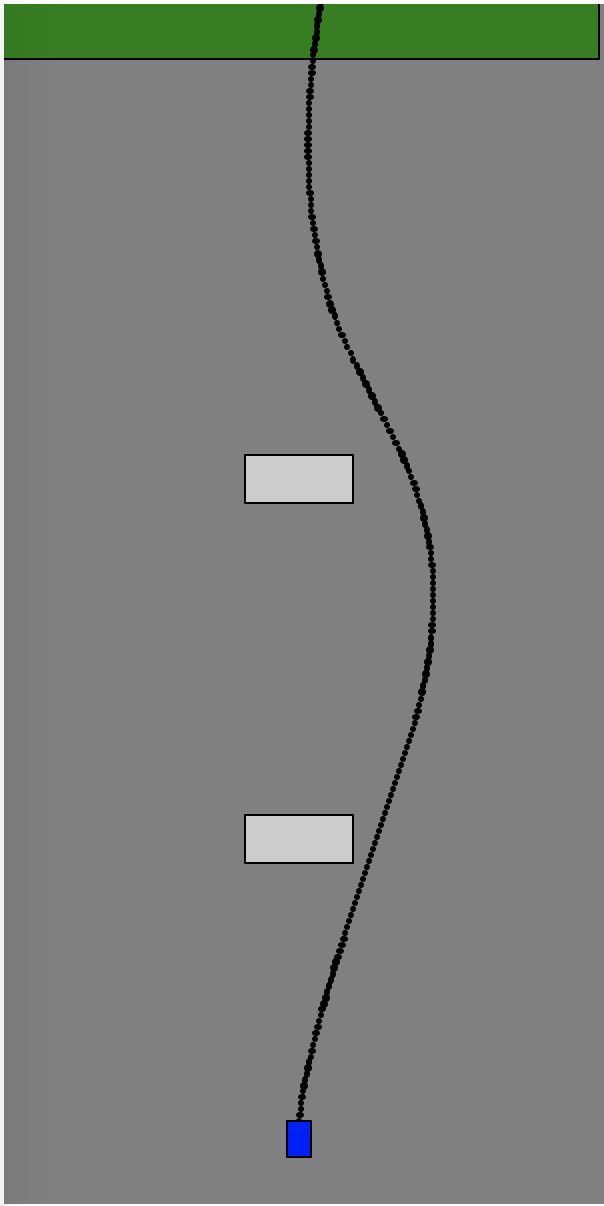}}
    \caption{Figure of the expert trajectory for different multi-obstacle environments.}
    \label{app:fig:expert_multiobj}
\end{figure*}

The results for \emph{LL}$\rightarrow$\emph{RL} and \emph{RR}$\rightarrow$\emph{LR} are not similar. In fact, the obstacles and the initial position of the car are at the center, the reward function for \emph{LL} and \emph{RR} are symmetrical, and the reward function for \emph{RL} and \emph{LR} are symmetrical. However, we randomly initialize the policy and learn a stochastic policy by reinforcement learning as the optimal policy for each reward function. As illustrated in Fig.~\ref{app:fig:expert_multiobj}, the trajectories for \emph{LL} and for \emph{RR} are asymmetrical due to this stochasticity. Similarly, the trajectories for $LR$ and for \emph{RL} are asymmetrical. 

For \emph{LL}$\rightarrow$\emph{RL} task, the proposed approach performs much better than Random while for \emph{RR}$\rightarrow$\emph{LR} tasks, Random performs better than the proposed method. We would like to emphasize that the divergence between \emph{LL} and \emph{RL} policy is smaller than the divergence between $RR$ and \emph{LR} policy including the distance of the final goal. So in this particular case, it turns out that transfer from \emph{RR} to \emph{LR} is more difficult than \emph{LL} to \emph{RL}, while training from a random initialization does not suffer from this difficulty that arises from the stochasticity of learning an initial policy.

\prg{2. Fetch Reach}

\textit{Datasets \& Environment.}
We use OpenAI Gym's Fetch Reach environment~\cite{brockman2016openai} that features a 7-DoF Fetch robotics arm. Actions are 8-dimensional: 3 dimensions specify the desired end effector movement in Cartesian coordinates, 4 dimensions specify the rotation of the end effector (quaternion), and the last dimension controls opening and closing the gripper. Observations include the Cartesian position of the end effector and its positional velocity, as well as the quaternion rotation of the end effector and its rotational velocity. The reward function is a linear combination of the distance to the goal, the homotopy class reward, barrier collision penalty, and a goal completion reward. Distance to the goal is defined by one dimension in $\mathbb{R}^3$ that determines whether the Fetch manipulator has extended its end effector far enough forward. The homotopy class reward is defined by another dimension in $\mathbb{R}^3$ that describes how far right or left Fetch extends its arm while reaching the goal. The homotopy class reward is annealed over time so that it is weighted less in later timesteps. We penalize barrier collisions with -10 and reward goal completion with +10. In Table 3 of the main paper, we then report the average number of interaction steps it takes to reach a return > 10.

\textit{Models \& Algorithms.}
For the \emph{L}$\rightarrow$\emph{R} task, we use a single-step curriculum where we remove the barrier and then re-introduce it. For the \emph{R} $\rightarrow$ \emph{L} task, the single-step curriculum did not work well, so we use a reward-weight curriculum where barrier collisions were penalized in the following $\alpha$ sequence: [0, 0.1, 0.5, 1.0]. We hypothesize that the asymmetry between the two tasks is caused by joint constraints in the Fetch manipulator; the asymmetry is a good example of our problem setting where some tasks may be harder to fine-tune to than others. 

\textit{Experimental Results.} We train policies with HER using the Stable Baselines implementation~\cite{stable-baselines}. We train for a total of 512,000 timesteps and evaluate our policy every 10,000 steps. Training a policy takes around 4 hours. We use the default hyperparameters in the Stable Baselines implementation which can be found \href{https://stable-baselines.readthedocs.io/en/master/modules/her.html#stable_baselines.her.HER}{here}. In Table 3 of the main paper, we average results over 5 runs. To generate runs for our method, we sample policies in the penultimate curriculum stage instead of using source policies. We found that training source policies in the Left and Right environments were challenging. Therefore, we trained our source policies in a relaxed environment where there was no barrier, or had a reduced cost of collision (-1 instead of -10). Since we assume that the source policy is given in our problem setting, it is okay to obtain in a relaxed setting. Finally, often found that BSS often ran into numerical issues that ended the training early; we consider these to be unsuccessful instances of fine-tuning.

\prg{3. Mujoco Ant}
\begin{figure}[h!]
    \centering
    \includegraphics[scale=0.58]{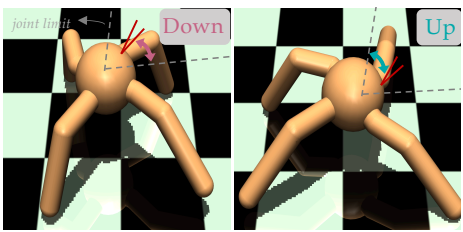}
    \caption{Mujoco Ant environment with a non-physical barrier. The red lines are the barrier states, or the set of joint angles the leg cannot move to. The grey dotted lines represent the upper right leg's joint limits.}
    \label{app:fig:ant}
\end{figure}

\textit{Datasets \& Environment.}
We experiment with Mujoco Ant environment~\cite{todorov2012mujoco} in OpenAI Gym~\cite{brockman2016openai}. In the Ant environment, as shown in Fig.~\ref{app:fig:ant}, we have an ant with four legs where each leg consists of two links and two joints. The state is represented as the states of all joint angles and the location and velocity of the center ball. The actions are 8-dimensional, which correspond to the rotation of $4\times2$ joints on the ant. The goal for the ant is to move forward in the direction of the x-axis in the environment as fast as possible. So the reward function for the original Ant environment corresponds to the distance traveled in the x-axis, the joint action cost, the force between links, and whether the ant survives. We create different homotopy classes by restricting the rotation angles of the upper right leg. We set the obstacle states as a set of joint angles $[\frac{\pi}{4}-0.2, \frac{\pi}{4}+0.2]$, which means that if the upper right leg rotates within this range, it will receive a large negative reward. We then create two homotopy classes: (1) the reward is higher if the joint angle of the upper right leg is larger than $\frac{\pi}{4}+0.2 \text{rad}$, denoted by \emph{Down}; (2) the reward is higher if the joint angle of the upper right leg is smaller than $\frac{\pi}{4}-0.2 \text{rad}$, denoted by \emph{Up}. We only perform the transfer task from \emph{Up} to \emph{Down} since the opposite direction is easy for all the baselines and the performance is not differentiable across methods.

\textit{Models \& Algorithms.}
For the reward-weight curriculum, we use a sequence of $\alpha$ as $0.001, 0.01, 0.1, 1$, where $\alpha$ multiplies the term related to the negative reward of the obstacle states. For the barrier-size curriculum, we use a sequence of obstacle angle ranges as $[\frac{\pi}{4}-0.002, \frac{\pi}{4}+0.002]$, $[\frac{\pi}{4}-0.02, \frac{\pi}{4}+0.02]$, and $[\frac{\pi}{4}-0.2, \frac{\pi}{4}+0.2]$. 

\textit{Experimental Results.}
We train the relaxed policy and fine-tune the policy at each curriculum step by TRPO algorithm implemented in PyTorch framework\footnote{https://github.com/ikostrikov/pytorch-trpo}. We train for total 10,000,000 environment steps, with 5000 environment steps per update to ensure convergence, where we can early stop if the policy achieves 1100+-100 reward. Training a policy takes around 1 hour. We evaluate our policy every 25,000 environment steps. We average our results over $10$ runs. To generate these runs, we initialize the policy parameters with different random seeds.

\prg{3. Assistive Gym}

\begin{figure*}
    \centering
    \includegraphics[width=\textwidth]{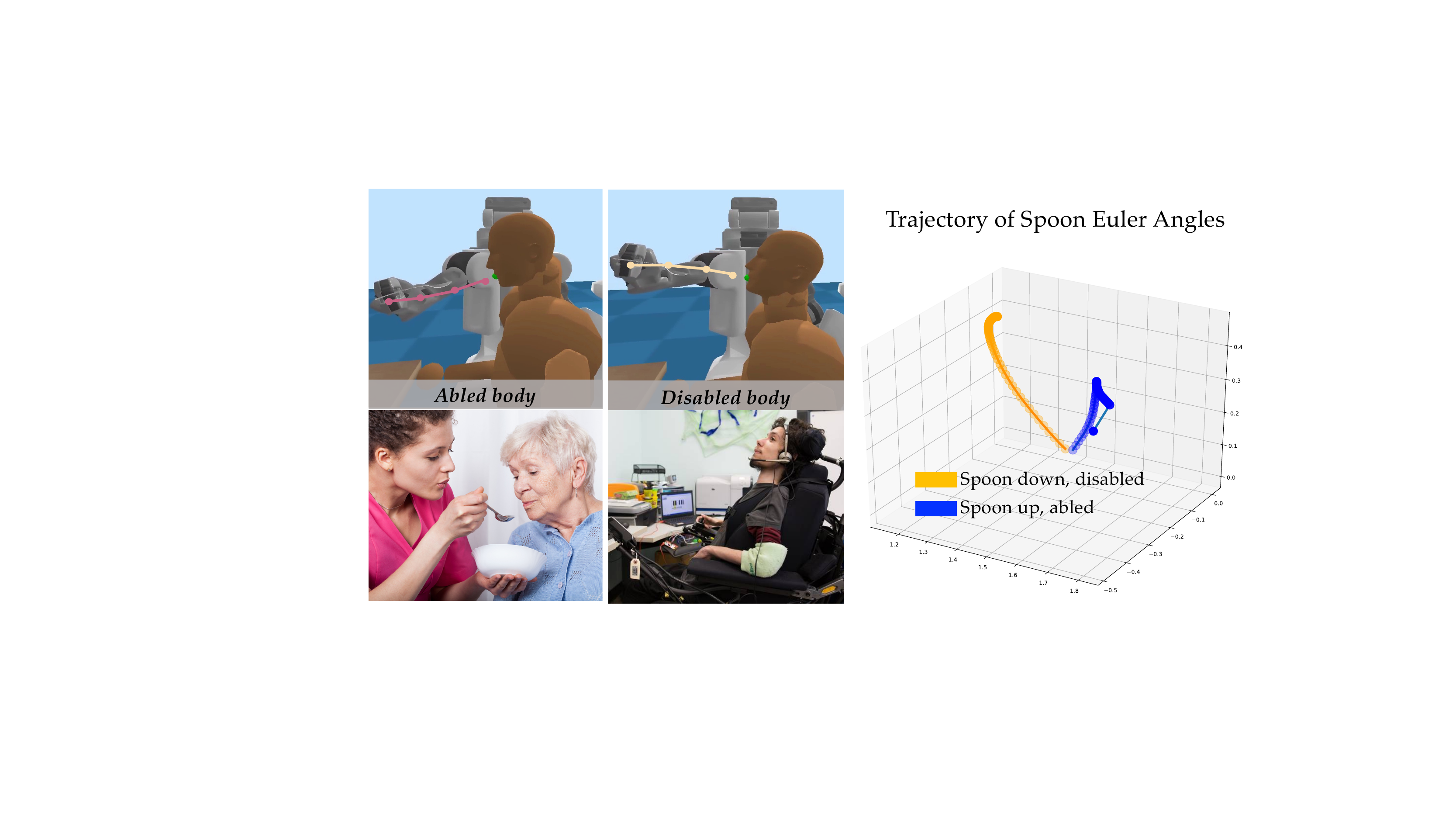}
    \caption{On the upper left, we plot the spoon trajectories in our assistive feeding domain for both abled bodies and disabled bodies. The corresponding images on the bottom show real-life examples of abled (lower left, \href{https://thecounter.org/safety-violations-nursing-home-kitchens-endanger-residents/}{source}) and disabled bodies (lower right, \href{https://equalengineers.com/mental-health-role-model-campaign/}{source}). The disabled patient is pictured in a wheelchair. It is common for patients to have their head back in order to control the head tracking and pointing device. The senior woman is an exampled of an abled body who points her gaze slightly downward as she is being fed. We find that by training with abled bodies, we can quickly learn a good policy for feeding disabled bodies. On the right, we show the trajectory of evolving spoon angles over time. The spoon down trajectory corresponds to feeding the disabled person whereas spoon up corresponds to feeding the abled person.}
    \label{app:fig:assistive}
\end{figure*}

\textit{Datasets \& Environment.}
Assistive Gym\footnote{https://github.com/Healthcare-Robotics/assistive-gym}~\cite{erickson2020assistivegym} is a physics-based simulation framework for physical human-robot interaction and robotic assistance. We implement our experiment under the feeding task, which uses a PR2 robot with 7 DoF manipulator. For the observation space, we use the distance from the spoon to the mouse, the spoon orientation, robot joint positions, and forces. For the action space, we use the joint action of the robot arm. As introduced in the main text, in this environment, we have two homotopy classes. One is feeding an abled person and the other is feeding a disabled person. The non-physical barrier is represented by a set of states. We let the spoon angle to feed  an abled person be $\theta_1$ and the angle for a disabled person as $\theta_2$. Then the angle between $\theta_1$ and $\theta_2$ is not good for feeding. To decide on the barrier states, we need to consider the following issues: 1) when the spoon is far away from the human mouth, we do not need the robot arm to hold the spoon up or down to fit the human mouth. We do not punish the robot for the wrong joint angle when the spoon is far away from the human mouth. 2) When the spoon is close to the human mouth, it is dangerous to hold spoon in the incorrect direction since it may pour the food on human body or face. Therefore, the states where the spoon is farther than $d_0$ from the human mouth are not in barrier state set. When the distance $d$ is smaller than $d_0$, we set the barrier states as the spoon angle in $[\theta_0-\frac{d_0-d}{d_0}\nabla\theta, \theta_0+\frac{d_0-d}{d_0}\nabla\theta]$, where $\theta_0$ is the median of the angle $\theta_1$ and $\theta_2$ and $\nabla\theta=\lvert\theta_1-\theta_2\rvert / 2$ is the half of the barrier angle range. The closer the spoon is to the mouth, the larger the size of barrier state set. We also give a homotopy reward if the spoon points to the correct homotopy direction during feeding (i.e. up for abled person and down for disabled person). The robot will receive a $0$ reward or otherwise a negative reward. We do not give positive rewards to ensure that the robot arm finishes the feeding tasks in the shortest time. Other parts of the reward function follow the original reward design in Assistive Gym environment. The final reward for a trajectory is $80\pm5$ if the robot arm successfully feeds the person and the trajectory is in the correct homotopy class. So we compute the number of interaction steps as the first policy achieving $80\pm5$ reward in the training process.

\textit{Models \& Algorithms.}
For the curriculum, we find that we only need one step curriculum, that is directly fine-tuning from relaxed policy to the target reward. 

\textit{Experimental Results.}
We train the relaxed policy and curriculum policies with PPO implemented in PyTorch framework. The python package is called $a2c\_ppo\_acktr$. We train for total 20,000,000 environment steps with 2400 environment steps per update to ensure convergence, where we can do early stop if the policy achieves $80\pm5$ reward. Training a policy takes around 12 hours. We evaluate our policy every 24,000 environment steps. In Table 3 of the main text, in the supplementary materials, we average results over 10 runs. To generate these runs, we sample three policy $\pi_s$ for abled person learned by initialized with different random seeds.

\prg{4. Lunar Lander}

\textit{Datasets \& Environment.}
We designed the reward function as follows: if the lander collides with the barrier in the middle, it will receive a $-100$ reward. The lander receives a reward in $[-5,5]$ at each time step based on which side of barrier it is in: if it is on the correct side of the barrier according to the homotopy class, it will receive higher reward. If the lander lands between two flags, it will receive an additional $+100$ reward. The remaining part of the reward follows the official Lunar Lander game: If it successfully lands on the surface safely, it will receive $+100$ reward. Otherwise, it will receive a $-100$ reward. The lander is punished by a potential field based the distance to the surface and the power consumed. All in all, the reward is higher than $+250$ for a correct trajectory in the homotopy class. We compute the number of interaction steps it takes for the policy to first achieve a reward greater than $250$. For each curriculum, we determine convergence based on whether the policy achieves $>250$ reward.

Here the barrier set size can also be defined by an weight $\beta$ which ranges from $0$ to $1$ to increases the barrier size from $0$ to the original size. Similar to our Navigation 2 experiment, we design the curriculum by first divide the hyper-parameter range and further divide the sub-range that is difficult for fine-tuning. 

\textit{Models \& Algorithms.}
The exact curriculum for barrier set size is: $[0.2,0.4,0.6,1.0]$.
The exact curriculum for reward weight is: $[0.1,0.4,0.7,1.0]$.

\textit{Experimental Results.}
The training details are the same as the Navigation 2 experiment.

\end{document}